\documentclass{article}
\usepackage[final]{neurips_2025}

\usepackage[utf8]{inputenc} % allow utf-8 input
\usepackage[T1]{fontenc}    % use 8-bit T1 fonts
\usepackage{hyperref}       % hyperlinks
\usepackage{url}            % simple URL typesetting
\usepackage{booktabs}       % professional-quality tables
\usepackage{amsfonts}       % blackboard math symbols
\usepackage{wrapfig,lipsum}
\usepackage{amsthm,mathrsfs,amsmath,amssymb}
\usepackage{caption}
\usepackage{subcaption} % graphicx,
\usepackage{graphicx,graphics}
\usepackage{svg}
\usepackage[noend]{algpseudocode}  % Eviter les endif, endfor
\usepackage{algorithm}
\usepackage{algorithmicx}
\usepackage{bbm}
\usepackage{bm}
\usepackage{tikz}
\usetikzlibrary{3d, calc}
\usepackage{svg}
\usepackage{xspace}
\graphicspath{{figures/}}
\usepackage[capitalize,noabbrev]{cleveref}
\usepackage{xcolor}
\usepackage[table,dvipsnames]{xcolor}
\usepackage{paralist} %you get compactitem
\usepackage{multirow, makecell, colortbl, hhline}
\usepackage{dirtytalk}
\newtheorem{theorem}{Theorem}[section]

\newtheorem{lemma}[theorem]{Lemma}

\newcommand{\ourmethod}{\texttt{3BASiL}\xspace}

\newcommand{\slr}{$(\bfS + \bfL\bfR)$\xspace}

\definecolor{aurometalsaurus}{rgb}{0.43, 0.5, 0.5}
\definecolor{britishracinggreen}{rgb}{0.0, 0.26, 0.15}
\definecolor{burntumber}{rgb}{0.54, 0.2, 0.14}
\definecolor{cobalt}{rgb}{0.0, 0.28, 0.67}
\definecolor{bulgarianrose}{rgb}{0.28, 0.02, 0.03}
\definecolor{ceruleanblue}{rgb}{0.16, 0.32, 0.75}
\definecolor{darkgreen}{RGB}{0,128,0}
\definecolor{darkred}{rgb}{0.74, 0.2, 0.24}

\newcommand{\bfW}{\mathbf{W}}

\newcommand{\bfS}{\mathbf{S}}
\newcommand{\bfL}{\mathbf{L}}
\newcommand{\bfR}{\mathbf{R}}

\newcommand{\rk}[1]{\operatorname{rank}\left({#1}\right)}

\newcommand{\argmin}{\operatornamewithlimits{\arg\min}}

\newtheorem{thm}{Theorem}

\usepackage[dvipsnames]{xcolor}
\usepackage{tikz}
\usetikzlibrary{shapes.geometric, arrows}
\usetikzlibrary{positioning}
\usetikzlibrary {arrows.meta}
\usetikzlibrary{decorations.pathreplacing,calligraphy}
\usepackage{enumitem}

\title{%
  \begin{minipage}[c]{0.1\textwidth}
  \includegraphics[width=0.9\linewidth]{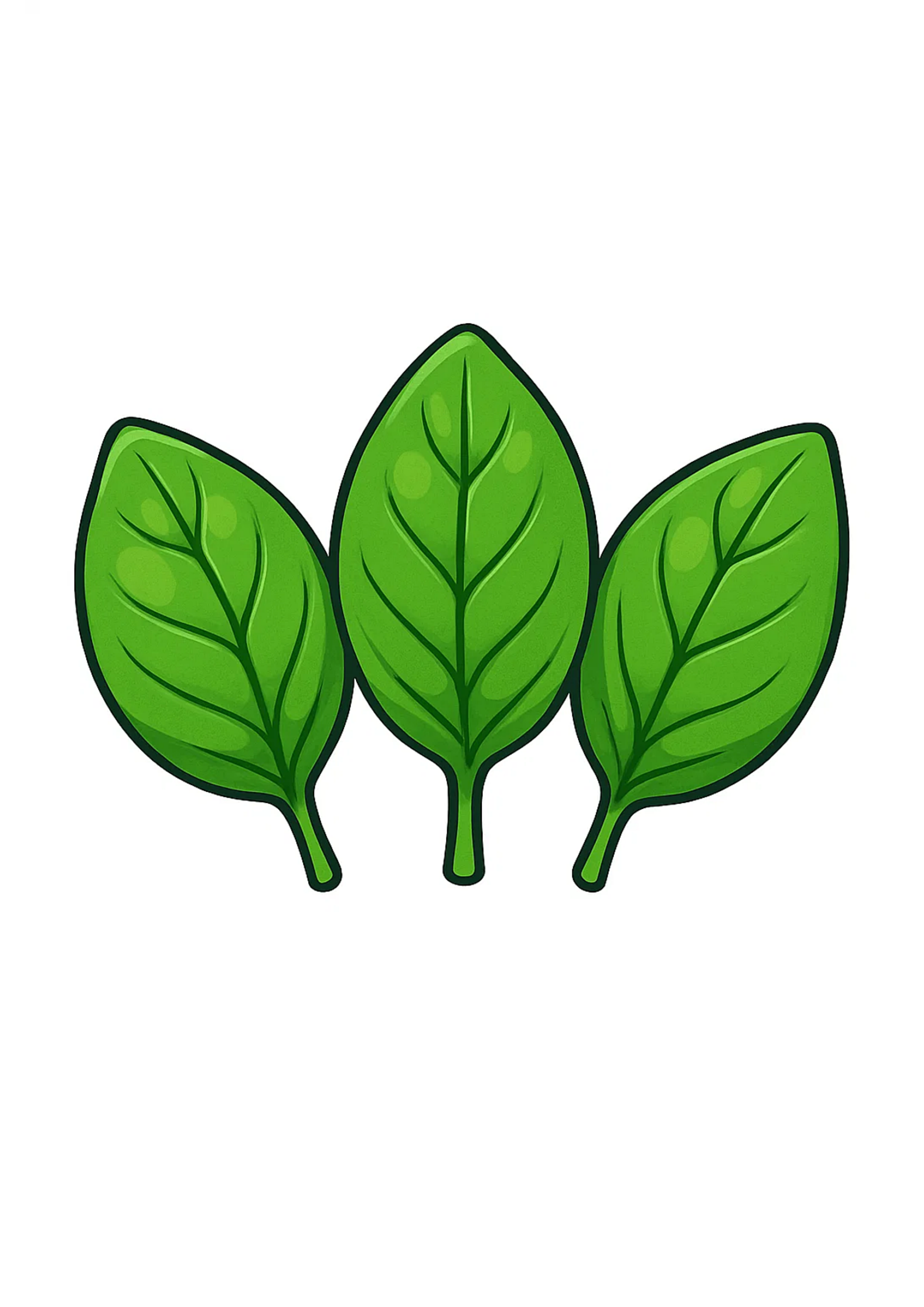}
  \end{minipage}
  \begin{minipage}[c]{0.85\textwidth}
  3BASiL: An Algorithmic Framework for Sparse plus Low-Rank Compression of LLMs
  \end{minipage}
}

\author{%
  Mehdi Makni \\
  Operations Research Center\\
  Massachusetts Institute of Technology\\
  \texttt{mmakni@mit.edu} \\
  \And
  Xiang Meng \\
  Operations Research Center\\
  Massachusetts Institute of Technology\\
  \texttt{mengx@mit.edu} \\
  \And
  Rahul Mazumder \\
  Operations Research Center\\
  Massachusetts Institute of Technology\\
  \texttt{rahulmaz@mit.edu} \\
}

\begin{document}
\maketitle
\begin{abstract}
Sparse plus Low-Rank $(\mathbf{S} + \mathbf{L}\mathbf{R})$ decomposition of Large Language Models (LLMs) has emerged as a promising direction in \textit{model compression}, aiming to decompose pre-trained model weights into a sum of sparse and low-rank matrices $\bfW \approx \bfS + \bfL\bfR$. Despite recent progress, existing methods often suffer from substantial performance degradation compared to dense models. In this work, we introduce \texttt{3BASiL-TM}, an efficient one-shot post-training method for $(\mathbf{S} + \mathbf{L}\mathbf{R})$ decomposition of LLMs that addresses this gap. Our approach first introduces a novel 3-Block Alternating Direction Method of Multipliers (ADMM) method, termed \texttt{3BASiL}, to minimize the layer-wise reconstruction error with convergence guarantees. 
We then design an efficient transformer-matching (\texttt{TM}) refinement step that jointly optimizes the sparse and low-rank components across transformer layers. This step minimizes a novel memory-efficient loss that aligns outputs at the transformer level.
Notably, the \texttt{TM} procedure is universal as it can enhance any $(\bfS + \bfL\bfR)$ decomposition, including pure sparsity. Our numerical experiments show that \texttt{3BASiL-TM} reduces the WikiText2 perplexity gap relative to dense LLaMA-8B model by over $30\%$ under a (2:4 Sparse + 64 LR) configuration, compared to prior methods. Moreover, our method achieves over 2.5x faster compression runtime on an A100 GPU compared to SOTA $(\mathbf{S} + \mathbf{L}\mathbf{R})$ method. 
Our code is available at {\small{\textcolor{blue}{\url{https://github.com/mazumder-lab/3BASiL}}}}.
\end{abstract}

\vspace{-1.0em}
\section{Introduction}
\label{sec:introduction}
\vspace{-0.8em}
\input{Tex/sec-introduction}

\vspace{-0.8em}
\section{Highly effective Sparse plus Low-Rank decomposition via ADMM}
\label{sec:3blockadmm}
\vspace{-0.8em}
% !TEX root =../main.tex

\subsection{Problem formulation}
\vspace{-0.5em}
We compress the layers of an LLM sequentially, one at a time by minimizing the reconstruction error between the outputs of pre-trained weights and compressed ones on a set of given input activations.
Formally, let $\widehat{\mathbf{W}}$ represent the pre-trained weight matrix of a given layer, and $\mathbf{X}$ denote the input activations (i.e., output of previous layers) on a set of $N$ calibration samples. In our setting, the goal of layer-wise reconstruction is to find a \slr decomposition that minimizes the $\ell_2$ error between the outputs of the original and decomposed weights---this can be formulated as follows:
\begin{align}\label{eq:original}
\min_{\mathbf{S},\mathbf{L}} \quad & \frac{1}{2}\left\|\mathbf{X}\widehat{\mathbf{W}} - \mathbf{X}\left(\mathbf{S}+\mathbf{L}\right)\right\|_F^2+ \frac{\lambda}{2} \left\|\widehat{\mathbf{W}}- (\mathbf{S}+\mathbf{L})\right\|_F^2~~~~\text{s.t. } ~~~\mathbf{S}\in\mathcal{S},\quad \rk{\mathbf{L}}\le r.
\end{align}
Above $\|\cdot\|_F$ denotes the Frobenius norm, $\mathcal{S}$ denotes the set of matrices satisfying a specified sparsity constraint (e.g., unstructured sparsity with given sparsity level or N:M sparsity); $\mathbf{S}$ and $\mathbf{L}$ denote the sparse and low-rank components, respectively. Parameter $\lambda > 0$ encourages the decomposed weights to remain close to the pre-trained ones.
\vspace{-0.5em}
\subsection{A multi-Block ADMM approach for layer-wise reconstruction}
\vspace{-0.5em}
The primary challenge in optimizing problem \eqref{eq:original} lies in the joint minimization of $\mathbf{S}$ and $\mathbf{L}$ under two complex constraints---sparsity and low-rank. To address this, we employ the Alternating Direction Method of Multipliers (ADMM), which enables separate updates of $\mathbf{S}$ and $\mathbf{L}$ at each iteration while maintaining their interdependence through a Lagrangian multiplier. This approach preserves the power of joint optimization while making the problem tractable. Our 3-block ADMM introduces an auxiliary variable $\mathbf{D}$ as a copy of the sparse component $\mathbf{S}$, reformulating problem \eqref{eq:original} as:
\begin{align}\label{eq:admm}
\min_{\mathbf{S},\mathbf{D},\mathbf{L}} \quad &\frac{1}{2}\left\|\mathbf{X}\widehat{\mathbf{W}} - \mathbf{X}\left(\mathbf{S}+\mathbf{L}\right)\right\|_F^2+ \frac{\lambda}{2} \left\|\widehat{\mathbf{W}}- (\mathbf{S}+\mathbf{L})\right\|_F^2 + \mathbb{I}_{\mathcal{S}}(\mathbf{D}) \\
\text{s.t. } \quad& \mathbf{S}=\mathbf{D},\quad \rk{\mathbf{L}}\le r.\notag
\end{align}
where $\mathbb{I}_{\mathcal{S}}(\mathbf{D})$ is an indicator function that equals to infinity when $\mathbf{D}\notin\mathcal{S}$ and zero otherwise. The augmented Lagrangian function with dual variable $\mathbf{V}$ and a quadratic penalty parameter $\rho>0$ reads:
\begin{align*}
\mathcal{L}_\rho( \mathbf{S},\mathbf{L}, \mathbf{D}, \mathbf{V})
= \frac{1}{2}\left\|\mathbf{X}\widehat{\mathbf{W}} - \mathbf{X}\left(\mathbf{S}+\mathbf{L}\right)\right\|_F^2+ \frac{\lambda}{2} \left\|\widehat{\mathbf{W}}- (\mathbf{S}+\mathbf{L})\right\|_F^2 +\mathbb{I}_{\mathcal{S}}(\mathbf{D}) + \frac{\rho}{2} \left\|\mathbf{S}-\mathbf{D}+\frac{\mathbf{V}}{\rho}\right\|_F^2.
\end{align*}
The method forms updates from the augmented Lagrangian with respect to three variables sequentially: the sparse component $\mathbf{S}$, the low-rank component $\mathbf{L}$, and sparse component's constrained copy $\mathbf{D}$, followed by a dual update (in variable $\mathbf{V}$). This sequential optimization over three variable blocks gives the method its name: 3-Block ADMM. At iteration $t$, the subproblem solutions are given below; the low-rank solution is a candidate subject to the acceptance rule specified below:
\begin{alignat*}{2}
&\mathbf{S}^{(t+1)} = \argmin\nolimits_{\mathbf{S}} \mathcal{L}_{\rho}( \mathbf{S},\mathbf{L}^{(t)}, \mathbf{D}^{(t)}, \mathbf{V}^{(t)})\quad &\mathbf{L}^{\star} &= \argmin\nolimits_{\mathbf{L}} \mathcal{L}_{\rho}( \mathbf{S}^{(t+1)}, \mathbf{L}, \mathbf{D}^{(t)}, \mathbf{V}^{(t)}) \\
&\mathbf{D}^{(t+1)} = \argmin\nolimits_{\mathbf{D}} \mathcal{L}_{\rho}( \mathbf{S}^{(t+1)},\mathbf{L}^{(t+1)}, \mathbf{D}, \mathbf{V}^{(t)})\quad& \mathbf{V}^{(t+1)}& = \mathbf{V}^{(t)} + \rho(\mathbf{S}^{(t+1)}-\mathbf{D}^{(t+1)}).
\end{alignat*}
Below, we derive the updates. For notational simplicity, we denote $\mathbf{H}=\mathbf{X}^\top \mathbf{X} +\lambda \mathbf{I}$.

\noindent\textbf{$\mathbf{S}$-block update}~~
Since $\mathcal{L}_{\rho}(\mathbf{S},\mathbf{L}^{(t)}, \mathbf{D}^{(t)}, \mathbf{V}^{(t)})$ is a quadratic function of $\mathbf{S}$, we obtain the closed-form solution by setting the gradient to zero:
\begin{equation}
\mathbf{S}^{(t+1)} = \left(\mathbf{H}+\rho \mathbf{I}\right)^{-1} \left(\mathbf{H}(\widehat{\mathbf{W}}-\mathbf{L}^{(t)}) - \mathbf{V}^{(t)} +\rho \mathbf{D}^{(t)}\right).
\end{equation}
\noindent\textbf{$\mathbf{L}$-block update}~~
Note that the $\mathbf{L}$-optimization subproblem can be reformulated as minimizing $\|\mathbf{H}^{1/2}(\widehat{\mathbf{W}}-\mathbf{S}^{(t+1)}-\mathbf{L})\|_F^2$ subject to the rank constraint. When $\mathbf{H}$ is full-rank (satisfied for any $\lambda > 0$), this problem has the closed-form solution (see \cref{sec:related-work} for a discussion about rank-reduced regression results) :
\begin{equation}\label{eq:L-update}
\mathbf{L}^{\star} =\mathbf{H}^{-1/2} P_r\left(\mathbf{H}^{1/2} (\widehat{\mathbf{W}}- \mathbf{S}^{(t+1)})\right),
\end{equation}
where $P_r$ denotes the best rank-$r$ approximation, which can be computed via SVD.\footnote{The closed-form solution in \cref{eq:L-update} can be used to improve other \slr methods like HASSLE-free (which employs gradient-descent on a reparameterized $\mathbf{L}=\mathbf{UV}^\top$)}

We accept $\mathbf{L}^{\star}$ as $\mathbf{L}^{(t+1)}$ only if $\|\mathbf{L}^{\star}-\mathbf{L}^{(t)}\|_F\le C_L\|\mathbf{S}^{(t+1)}-\mathbf{S}^{(t)}\|_F$, for a fixed finite $C_L\ge0$; otherwise, we retain $\mathbf{L}^{(t)}$.

\noindent\textbf{$\mathbf{D}$-block update}~~
The optimization over $\mathbf{D}$ involves projecting $\mathbf{S}^{(t+1)} + \mathbf{V}^{(t)}/\rho$ onto the sparsity constraint set $\mathcal{S}$, which corresponds to magnitude-based pruning of $(\mathbf{S}^{(t+1)} + \mathbf{V}^{(t)}/\rho)$---we sort $[(\mathbf{S}^{(t+1)} + \mathbf{V}^{(t)}/\rho)_{ij}]^2$ and retain only the largest values. For unstructured pruning, we keep a predetermined fraction of the largest values; for N:M structured sparsity, we retain N largest values out of every M consecutive weights. %This projection is denoted as $P_{\mathcal{S}}(\cdot)$.

% In practice, we use a non-decreasing iteration-dependent penalty parameter $\rho_t$ (see \cref{sect:appendix-exp-details}).
With an iteration-dependent penalty $\rho_t$, the updates and low-rank candidate are given below; $\mathbf{L}^{(t+1)}$ is obtained by the acceptance rule above:
\begin{alignat}{2}\label{eq:finalupdate}
&\mathbf{S}^{(t+1)}\!=\!\left( \mathbf{H}+\rho_t \mathbf{I}\right)^{-1} ( \mathbf{H}(\widehat{\mathbf{W}}-\mathbf{L}^{(t)}) - \mathbf{V}^{(t)} +\rho_t \mathbf{D}^{(t)})\,\, &\mathbf{L}^{\star} &\!=\! \mathbf{H}^{-1/2} P_r(\mathbf{H}^{1/2} (\widehat{\mathbf{W}}- \mathbf{S}^{(t+1)})) \notag \\
&\mathbf{D}^{(t+1)} \!=\! P_\mathcal{S}(\mathbf{S}^{(t+1)} + \mathbf{V}^{(t)}/\rho_t)\,\,& \mathbf{V}^{(t+1)}& \!=\! \mathbf{V}^{(t)} + \rho_t(\mathbf{S}^{(t+1)}-\mathbf{D}^{(t+1)}).
\end{alignat}

%  \begin{alignat}{2}\label{eq:finalupdate}
% &\mathbf{S}^{(t+1)}\!=\!\left( \mathbf{H}+\rho_t \mathbf{I}\right)^{-1} ( \mathbf{H}(\widehat{\mathbf{W}}-\mathbf{L}^{(t)}) - \mathbf{V}^{(t)} +\rho_t \mathbf{D}^{(t)})\,\, &\mathbf{L}^{(t+1)} &\!=\! \mathbf{H}^{-1/2} P_r(\mathbf{H}^{1/2} (\widehat{\mathbf{W}}- \mathbf{S}^{(t+1)})) \notag \\
% &\mathbf{D}^{(t+1)} \!=\! P_\mathcal{S}(\mathbf{S}^{(t+1)} + \mathbf{V}^{(t)}/\rho_t)\,\,& \mathbf{V}^{(t+1)}& \!=\! \mathbf{V}^{(t)} + \rho_t(\mathbf{S}^{(t+1)}-\mathbf{D}^{(t+1)}).
% \end{alignat}
\vspace{-1.0em}
\noindent \textbf{Computational complexity}~~~ We implement several tricks to reduce the computational cost in the $\mathbf{S}$ and $\mathbf{L}$-update steps,
which constitute the major computational cost of 3-Block ADMM algorithm. For the $\mathbf{S}$-update step, we adopt the approach of \cite{meng2024alps} by pre-computing (once) and storing the eigenvalue decomposition $\mathbf{H}=\mathbf{U}\mathbf{\Sigma}\mathbf{U}^\top$. This allows us to efficiently calculate the matrix inverse $(\mathbf{H} +\rho \mathbf{I})^{-1}=\mathbf{U} (\mathbf{\Sigma}+\rho \mathbf{I})^{-1}\mathbf{U}^\top$ for varying values of $\rho$ across iterations. For an efficient $\mathbf{L}$-update step, we store the matrices $\mathbf{H}^{-1/2}=\mathbf{U}\mathbf{\Sigma}^{-1/2}$, and $\mathbf{H}^{1/2}=\mathbf{\Sigma}^{1/2}\mathbf{U}^\top$ and employ a randomized-SVD procedure \citep{halko2011finding} for numerical efficiency.
In the context of LLMs, the weight matrices scale with the transformer's hidden dimension $N$. Our algorithm's per-iteration time complexity comprises: five matrix-matrix multiplications with complexity $O(N^3)$, a Randomized-SVD operation with complexity $O(N^2r)$ to enforce rank constraint (using constant oversampling and power iterations as in \citep{halko2011finding}), and a projection onto $\mathcal{S}$ requiring at most $O(N^2 \log(N))$ for sorting and thresholding operations—across the entire matrix for unstructured sparsity or within blocks for semi-structured sparsity. The overall time complexity is $O(N^3)$.

\subsection{Convergence of ADMM}
\vspace{-0.5em}
Despite its appeal and usage,
%%%%widespread application,
the convergence properties of 3-Block ADMM remain theoretically challenging. \cite{chen2016direct} demonstrated that without additional conditions, the algorithm may fail to converge, while later works \citep{lin2015sublinear,wang2018convergence} established various sufficient conditions for convergence.

Before applying the low-rank acceptance rule, the exact subproblem updates can be reformulated as a standard 2-block ADMM by treating $(\mathbf{L},\mathbf{D})$ as a single variable block. This reformulation is valid because the Lagrangian is separable with respect to $\mathbf{L}$ and $\mathbf{D}$, meaning their joint minimization yields equivalent updates to sequential optimization (although 3-blocks remain the \say{natural} way to conceptualize the updates). While \cite{meng2024alps} established convergence guarantees for ADMM applied to layerwise pruning, their analysis addresses a different problem formulation than ours. Specifically, they apply ADMM solely to unstructured pruning, whereas our approach extends to \slr decomposition. Our framework includes a low-rank component with relatively complex updates in each iteration, which introduces additional mathematical challenges in convergence analysis that prevent direct application of the results in \cite{meng2024alps}.

%We observe that our proposed 3-block ADMM approach can be reformulated as a standard 2-block ADMM by treating $(\mathbf{L},\mathbf{D})$ as a single variable block. This reformulation is valid because the Lagrangian is separable with respect to $\mathbf{L}$ and $\mathbf{D}$, meaning their joint minimization yields equivalent updates to sequential optimization (3-blocks are still the \say{natural} way to look at the updates). Nevertheless, the convergence analysis remains non-trivial due to the presence of two complex non-convex constraints—sparsity and low-rank requirements—a scenario not addressed by existing ADMM convergence results.

To address this gap, we establish the following novel convergence guarantee that ensures the decomposition converges as long as we choose penalty parameter $\rho_t$ that increases sufficiently rapidly (refer to \cref{sect:proofadmm} for a complete proof).
\begin{thm}\label{thm:admm}
Let $\left\{\mathbf{S}^{(t)}\right\}_{t=0}^\infty$ and $\left\{\mathbf{L}^{(t)}\right\}_{t=0}^\infty$ be the sequences generated by \eqref{eq:finalupdate} with the low-rank acceptance rule above. Suppose the penalty parameter $\rho_t$ chosen at iteration $t$ is positive, non-decreasing and satisfies $\sum_{t=0}^\infty 1/\rho_t < \infty$. Then for any $t\ge 1$:
\begin{equation}
\max\{\| \mathbf{S}^{(t+1)} - \mathbf{S}^{(t)} \|_F, \| \mathbf{L}^{(t+1)} - \mathbf{L}^{(t)} \|_F\} \le {C}/{ \rho_{t-1} },
\end{equation}
where $C$ depends on $\mathbf{X}$, $\widehat{\mathbf{W}}$, $\lambda$, $\rho_0$, $C_L$, the initialization, and $\sum_{t=0}^\infty 1/\rho_t$.
In particular, there exists a matrix $\mathbf{\bar W}$ such that $\mathbf{S}^{(t)}+\mathbf{L}^{(t)} \rightarrow \mathbf{\bar W}$ as $t \rightarrow \infty$.
\end{thm}

\vspace{-0.8em}
\section{Transformer-level matching}
\label{sec:transmatch}
\vspace{-0.8em}
% !TEX root = ../main.tex
After layer-wise pruning, LoRA can directly refine the low-rank components in the \slr decomposition for task adaptation. However, the sparse components are not well-optimized by this process, as they are determined solely via layer-wise objectives. These layer-wise objectives are imperfect proxies for the true end-to-end loss function. On the other hand, fully optimizing the sparse components using the true end-to-end loss is computationally expensive and requires a full back-propagation on the entire network.
To address this limitation, we introduce an efficient \textit{transformer-matching} refinement step that leverages transformer-level information to enhance the sparse components. This procedure is efficient because it requires comparable CUDA memory and runtime to the compression algorithms themselves.

Our transformer-matching procedure jointly optimizes all sparse and low-rank components across layers within a transformer block to better match the original transformer's output. It acts as an intermediate loss function between layer-wise proxies and the true end-to-end loss. This approach can enhance any \slr decomposition, including pruning (where $\mathbf{LR} = \mathbf{0}$). \cref{fig:transformer-matching-universality} illustrates the performance gains obtained after applying \texttt{TM} to state-of-the-art one-shot \slr decomposition algorithms. In \cref{tab:compression-rho}, we show results of applying \texttt{transformer-matching} to pruning algorithms with pure sparsity constraints like WandA \citep{sun2023simple}, SparseGPT \citep{frantar2023sparsegpt}, and ALPS \citep{meng2024alps} highlighted in dark red.

Formally, for each transformer block $T_i$ with $L$ layers, after obtaining sparse and low-rank components $\{\mathbf{S}^{(i,\ell)}, \mathbf{L}^{(i,\ell)}\}_{\ell=1}^{L}$ through layer-wise pruning, we denote the support of sparse components as $\mathcal{S}^{(i,\ell)} = \operatorname{Supp}(\mathbf{S}^{(i,\ell)})$. Let $\mathbf{X}_i$ represent the outputs from the previously compressed transformer block $T_{i-1}$. We then refine these components using a transformer-level reconstruction loss:
\begin{equation}\label{eq:transformer-matching}
\begin{aligned}
  \min_{\{\mathbf{S}^{(i,\ell)}, \mathbf{L}^{(i,\ell)}\}_{\ell=1}^{L}}\,\,\,\, &  \left\|T_i\left(\mathbf{X}_i; \{\mathbf{W}^{(i,\ell)}\}_{\ell=1}^{L}\right) - T_i\left(\mathbf{X}_i; \{\mathbf{S}^{(i,\ell)} + \mathbf{L}^{(i,\ell)}\}_{\ell=1}^{L}\right)\right\|_F^2,\\
 \text{s.t.}\qquad& \operatorname{Supp}(\mathbf{S}^{(i,\ell)}) \subset \mathcal{S}^{(i,\ell)}, \quad \text{rank}(\mathbf{L}^{(i,\ell)}) \leq r^{(i, \ell)}
\end{aligned}
\end{equation}
where this constraint optimizes the weights of the decomposed components. 
Due to the non-linear activations between layers, we use gradient-based optimization methods such as Adam. Nonetheless, this optimization remains computationally efficient as it is performed using iteratively chunks of the small calibration dataset used for compression. Additionally, the forward/backward passes are limited to only one transformer block.
The transformer-matching approach offers two key advantages. First, it creates a more accurate proxy of the original loss function by directly minimizing the discrepancy between the original and compressed transformer outputs, resulting in higher-performance pruned models. Second, it reduces accumulated errors—introduced in layer-wise pruning where input activations $\mathbf{X}$ are computed from outputs of previously \textit{pruned} layers—by ensuring that activations fed into subsequent layers more faithfully match those of the dense model:
\begin{equation}\label{eq:similar-activations}
 T_i\left(\mathbf{X}_i; \{\mathbf{S}^{(i,\ell)} + \mathbf{L}^{(i,\ell)}\}_{\ell=1}^{L}\right) = \mathbf{X}_{i+1} ~~\approx~~ \mathbf{X}^{\text{(oracle)}}_{i+1} = T_i\left(\mathbf{X}_i; \{\mathbf{W}^{(i,\ell)}\}_{\ell=1}^{L}\right),
\end{equation}
therefore providing better activation statistics for compression on subsequent transformers  compared to layer-wise reconstruction which only matches weight matrices layer by layer.

After transformer-matching, the refined sparse components $\mathbf{S}^{(i,\ell)}$ remain fixed during downstream fine-tuning, while the low-rank components $\mathbf{L}^{(i,\ell)}$  serve as smart initializations for efficient LoRA adaptation to specific tasks.
\begin{figure}[h]
    \centering
    \vspace{-0.75em}
    \includegraphics[width=\linewidth]{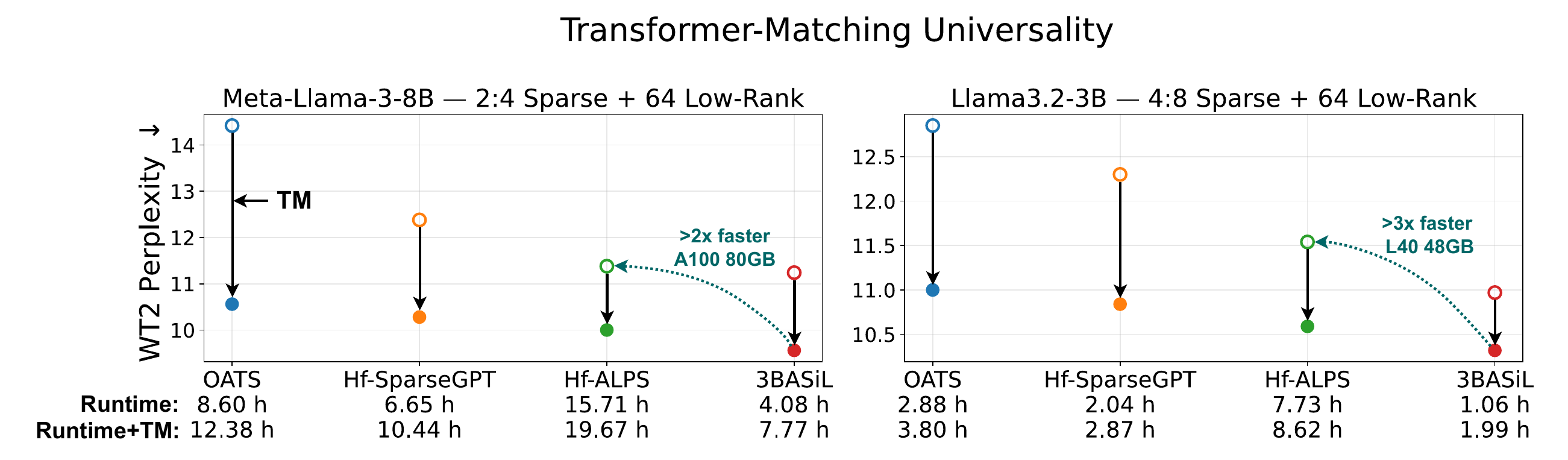}
    \caption{{\small{Our transformer matching (\texttt{TM}) procedure improves any one-shot \slr decomposition method (see baselines in \cref{sec:experimental-results}) with a small computational overhead. Circled markers represent standard \slr methods, while filled markers indicate their \texttt{TM}-enhanced versions. Black arrows illustrate performance gains due to \texttt{TM}. 
    The compression runtimes are reported in hours. Llama3-8B models were run on a A100 GPU, while Llama3.2-3B were run on a L40 GPU.
    Our proposal \texttt{3BASiL-TM}, remains significantly faster: \textbf{(left)} over 2$\times$ speedup on an A100 80GB for the Llama3-8B model decomposed to (2:4+64LR) configuration, and \textbf{(right)} over 3$\times$ speedup on an L40 48GB for the Llama3.2-3B model decomposed to (4:8+64LR) configuration (both compared to Hf-ALPS).}}}
    \label{fig:transformer-matching-universality}
    \vspace{-1.5em}
\end{figure}

\vspace{-0.5em}
\section{Experimental results}
\label{sec:experimental-results}
\vspace{-0.8em}
% !TEX root = ../main.tex
\subsection{Experimental setup} \label{sect:exp-setup}
\vspace{-0.5em}
\noindent\textbf{Models and LLM Evaluation Protocol}~~~ 
To rigorously assess the effectiveness of our proposed approach \ourmethod and \textit{transformer-matching} (\texttt{TM}) procedure, we conducted extensive experiments on the Llama-3 and Llama-3.2 model families \citep{dubey2024llama} and scaled results in one experiment to a OPT-30B \citep{opt} model, hence covering architectures with number of parameters ranging from 1B to 30B. Following the widely adopted setup introduced by \cite{frantar2023sparsegpt} for one-shot compression, we select the calibration set consisting of  128 randomly sampled text segments (2048 tokens each) from the C4 \citep{JMLR:v21:20-074} train dataset's first shard. This calibration set is shared across all evaluated compression methods to ensure consistency.

We adopt two evaluation criteria: (1) \textit{perplexity} as a foundational measure of language modeling quality, and (2) \textit{zero-shot task performance} to assess practical downstream capabilities post-compression. Perplexity is measured using three standard benchmarks:  WikiText2 \citep{merity2017pointer}, Penn Treebank \citep{Marcus1994}, and C4 validation samples, computed using HuggingFace’s full-stride perplexity protocol \citep{Perplexity}. For zero-shot evaluation, we utilize the LM Harness framework \citep{gao10256836framework} on a diverse suite of eight zero-shot tasks: PIQA \citep{bisk2020piqa}, ARC-Easy/Challenge \citep{clark2018think}, HellaSwag \citep{zellers2019hellaswag}, Winogrande \citep{sakaguchi2021winogrande}, RTE \citep{poliak2020survey}, OpenbookQA \citep{banerjee2019careful}, and BoolQ \citep{clark2019boolq}. We report individual scores for each benchmark and the average across all tasks. 
 
For perplexity, $(\downarrow)$ lower values are preferred. For zero-shot tasks, $(\uparrow)$ higher values are preferred.

\noindent\textbf{Baselines}~~ Our main baselines are OATS \citep{zhang2024oats}, HASSLE-free-SparseGPT (Hf-SparseGPT) and HASSLE-free-ALPS (Hf-ALPS)---the latter two use pruning approaches SparseGPT \citep{frantar2023sparsegpt} and ALPS \citep{meng2024alps}, respectively, in the sparsification step of the alternating minimization algorithm proposed by \cite{makni2025hasslefree}. 

For all these baselines, we follow the original configuration and perform $80$ steps of alternating minimization. For HASSLE-free methods, we propose an improved implementation that replaces their original parameterization of $\mathbf{L} = \mathbf{UV}^\top$ and gradient-based optimization with the closed-form solution provided in \cref{eq:L-update}. This modification leads to improved compression runtime and better downstream LLM evaluation metrics--see~\cref{tab:our-hassle-free}. Under this improved implementation, the method EoRA \citep{DBLP:journals/corr/abs-2410-21271}, which applies the update in \cref{eq:L-update} once after one round of compression, reduces to HASSLE-free (alternating minimization approach) with a number of iterations equal to one. EoRA is the fastest \slr method but underperforms HASSLE-free which uses more alternating minimization steps (default=80), and hence there is a large gap compared to our approach on most model/configuration settings. We show some results of EoRA in \cref{tab:our-hassle-free}.

More details on the implementation of \ourmethod, \texttt{TM} and the baselines (with improved implementation) are provided in \cref{sect:appendix-exp-details}.

\vspace{-0.75em}
\subsection{Numerical results}
\vspace{-0.75em}

Our evaluation focuses primarily on (N:M $+ \, \mathbf{LR}$) decompositions, which enable efficient GPU acceleration via specialized CUDA kernels \citep{mozaffari2024slope, makni2025hasslefree}. We evaluate both one-shot compression performance and downstream LoRA fine-tuning capabilities. Additionally, we demonstrate the generality of our approach through experiments with unstructured sparsity and integration with sparsity allocation methods.
The downstream LoRA experiments have been motivated by recent studies \citep{li2023loftq, guo2023lq, saha2024compressing} suggesting that decompositions of the form $\mathcal{C}(\mathbf{W}) + \mathbf{LR}$ are \underline{LoRA-aware}: i.e. low-rank components obtained from compression can act as \textit{smart initialization} to improve downstream LoRA fine-tuning. Further numerical experiments where we ablate on \texttt{TM} and LoRA fine-tuning for \slr methods can be found in \cref{sect:appendix-more-experiments}.

\noindent\textbf{One-shot (Sparse + LR) results}~~~  
We compare \ourmethod to prior \slr decomposition methods in the one-shot compression setting—i.e., without fine-tuning. \cref{tab:perplexity-comparison} reports results for LLaMA3.2 family under various (N:M + 64LR) configurations. \cref{tab:oneshot-nm-64} and \cref{fig:c4-before-lft} show results for similar configurations for the LLaMA3-8B model. \ourmethod reduces perplexity by up to 8\% compared to previous SOTA (due to better layer-wise reconstruction---see \cref{fig:attn-layers-loss} and \cref{fig:mlp-layers-loss}), with the \texttt{TM} step yielding further dramatic improvements of up to 40\% perplexity reduction. 

We also compare \slr decompositions with semi-structured pure pruning methods under a fixed compression ratio $\rho = 50\%$. Results in \cref{tab:compression-rho} show that \texttt{3BASiL-TM} achieves the best compression-performance trade-off under (3:8 + LR) configurations among different \slr methods. Additionally, we expand our \slr experiments to include a (2:4 + 112) configuration for OPT-30B model \citep{opt}. This configuration uses a 1.56\% Low-Rank Adapter (hidden size 7168). Under this configuration, \cite{mozaffari2024slope} report a 1.53x speedup as well as a 0.63x memory reduction compared to dense model. Results are reported in \cref{tab:opt30b-comparison}.

For unstructured sparsity configurations, we benchmark \ourmethod against prior (S + LR) methods on a "less aggressive" (50\% + 128) compression for both Llama3.2-1B and Llama3-8B models. \cref{tab:unstr_results} shows that our proposed method maintains its advantage even in this near-lossless configuration regime. We further evaluate 3BASiL under high sparsity ratios with (Unstructured + 64) configurations and demonstrate how our method integrates with the sparsity allocation method OWL \citep{yin2024outlier} for the Llama3-8B model---see \cref{tab:owl_comparison} in \cref{sect:appendix-exp-details}.

These results highlight the effectiveness and flexibility of our method \ourmethod.

\begin{table}[htbp]
\vspace{-4mm}
\centering
\begin{minipage}[t]{0.44\textwidth}
\captionof{figure}{One-shot C4 perplexity analysis of Llama3-8B under different (N:M + 64LR) configurations.}
\centering
\vspace{1.0em}
\includegraphics[width=\linewidth]{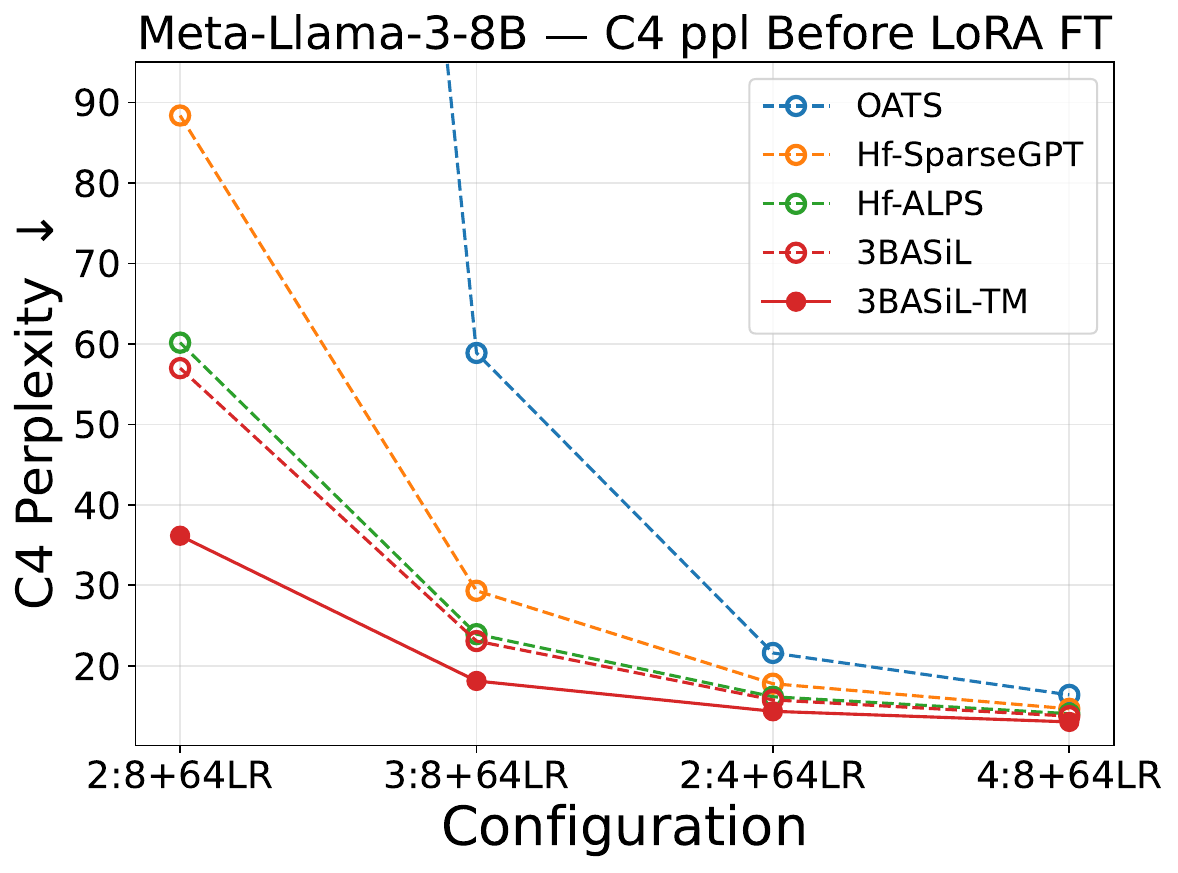}
\label{fig:c4-before-lft}
\end{minipage}%
\hfill
\begin{minipage}[t]{0.53\textwidth}
\vspace{0pt}
\centering
\captionof{table}{Perplexity of Llama-3.2 family}
\label{tab:perplexity-comparison}
\resizebox{\textwidth}{!}{%
\begin{tabular}{l|c|ccc|ccc}
\toprule
\multirow{2}{*}{\textbf{Method}} & \multirow{2}{*}{\textbf{Config}} & \multicolumn{3}{c|}{\textbf{Llama-3.2-1B}} & \multicolumn{3}{c}{\textbf{Llama-3.2-3B}} \\
\cmidrule{3-8}
 &  & \textbf{C4 $\downarrow$} & \textbf{WT2 $\downarrow$} & \textbf{PTB $\downarrow$} & \textbf{C4 $\downarrow$} & \textbf{WT2 $\downarrow$} & \textbf{PTB $\downarrow$} \\
\midrule
\texttt{OATS} & \multirow{5}{*}{2:8+64LR} & 640.86 & 605.20 & 779.86 & 531.47 & 494.31 & 674.71 \\
\texttt{Hf-SparseGPT} &  & 162.45 & 134.21 & 170.12 & 106.07 & 106.17 & 151.92 \\
\texttt{Hf-ALPS} &  & 107.14 & 94.71 & 124.17 & 69.96 & 65.34 & 108.68 \\
\cellcolor{darkgreen!10}\texttt{3BASiL} &  & \cellcolor{darkgreen!10}97.50 & \cellcolor{darkgreen!10}86.59 & \cellcolor{darkgreen!10}100.35 & \cellcolor{darkgreen!10}73.00 & \cellcolor{darkgreen!10}72.26 & \cellcolor{darkgreen!10}110.10 \\
\cellcolor{darkgreen!10}\texttt{3BASiL-TM} &  & \cellcolor{darkgreen!10}\textbf{55.24} & \cellcolor{darkgreen!10}\textbf{49.74} & \cellcolor{darkgreen!10}\textbf{69.49} & \cellcolor{darkgreen!10}\textbf{45.35} & \cellcolor{darkgreen!10}\textbf{42.38} & \cellcolor{darkgreen!10}\textbf{68.29} \\
\midrule
\texttt{OATS} & \multirow{5}{*}{3:8+64LR} & 125.91 & 92.13 & 115.80 & 65.08 & 47.27 & 81.29 \\
\texttt{Hf-SparseGPT} &  & 43.50 & 34.18 & 51.16 & 34.66 & 26.60 & 39.76 \\
\texttt{Hf-ALPS} &  & 37.80 & 29.00 & 43.60 & 27.94 & 22.77 & 34.59 \\
\cellcolor{darkgreen!10}\texttt{3BASiL} &  & \cellcolor{darkgreen!10}34.81 & \cellcolor{darkgreen!10}26.96 & \cellcolor{darkgreen!10}41.55 & \cellcolor{darkgreen!10}26.35 & \cellcolor{darkgreen!10}20.66 & \cellcolor{darkgreen!10}31.77 \\
\cellcolor{darkgreen!10}\texttt{3BASiL-TM} &  & \cellcolor{darkgreen!10}\textbf{26.26} & \cellcolor{darkgreen!10}\textbf{20.75} & \cellcolor{darkgreen!10}\textbf{32.09} & \cellcolor{darkgreen!10}\textbf{20.89} & \cellcolor{darkgreen!10}\textbf{17.18} & \cellcolor{darkgreen!10}\textbf{25.31} \\
\midrule
\texttt{OATS} & \multirow{5}{*}{4:8+64LR} & 28.06 & 19.69 & 32.90 & 19.25 & 13.40 & 21.67 \\
\texttt{Hf-SparseGPT} &  & 22.24 & 15.90 & 27.35 & 17.09 & 12.30 & 19.19 \\
\texttt{Hf-ALPS} &  & 20.71 & 14.90 & 24.75 & 16.04 & 11.51 & 18.17 \\
\cellcolor{darkgreen!10}\texttt{3BASiL} &  & \cellcolor{darkgreen!10}20.04 & \cellcolor{darkgreen!10}14.26 & \cellcolor{darkgreen!10}24.27 & \cellcolor{darkgreen!10}15.65 & \cellcolor{darkgreen!10}10.97 & \cellcolor{darkgreen!10}17.39 \\
\cellcolor{darkgreen!10}\texttt{3BASiL-TM} &  & \cellcolor{darkgreen!10}\textbf{18.66} & \cellcolor{darkgreen!10}\textbf{13.19} & \cellcolor{darkgreen!10}\textbf{22.46} & \cellcolor{darkgreen!10}\textbf{14.89} & \cellcolor{darkgreen!10}\textbf{10.29} & \cellcolor{darkgreen!10}\textbf{16.52} \\
\midrule
\texttt{OATS} & \multirow{5}{*}{2:4+64LR} & 41.80 & 28.45 & 45.36 & 25.18 & 17.41 & 28.60 \\
\texttt{Hf-SparseGPT} &  & 27.25 & 19.45 & 32.63 & 20.38 & 15.03 & 23.23 \\
\texttt{Hf-ALPS} &  & 23.90 & 17.66 & 28.96 & 18.45 & 13.79 & 20.50 \\
\cellcolor{darkgreen!10}\texttt{3BASiL} &  & \cellcolor{darkgreen!10}23.16 & \cellcolor{darkgreen!10}17.27 & \cellcolor{darkgreen!10}27.77 & \cellcolor{darkgreen!10}17.89 & \cellcolor{darkgreen!10}13.12 & \cellcolor{darkgreen!10}20.10 \\
\cellcolor{darkgreen!10}\texttt{3BASiL-TM} &  & \cellcolor{darkgreen!10}\textbf{20.46} & \cellcolor{darkgreen!10}\textbf{15.23} & \cellcolor{darkgreen!10}\textbf{24.60} & \cellcolor{darkgreen!10}\textbf{16.37} & \cellcolor{darkgreen!10}\textbf{11.79} & \cellcolor{darkgreen!10}\textbf{18.34} \\
\midrule
\rowcolor{NavyBlue!10}\textbf{Dense} & -- & 14.01 & 9.75 & 17.59 & 11.33 & 7.81 & 13.53 \\
\bottomrule
\end{tabular}%
}
\end{minipage}
\end{table}

\begin{table}
\vspace{-4mm}
\centering
\caption{One-shot (N:M Sparse + LR) decomposition performance for Meta-Llama-3-8B.}\label{tab:oneshot-nm-64}
\resizebox{\textwidth}{!}{%
\begin{tabular}{l|c|ccc|ccccccccc}
\toprule
\textbf{Method} & \textbf{Config} & \textbf{C4 $\downarrow$} & \textbf{WT2 $\downarrow$} & \textbf{PTB $\downarrow$} & \textbf{PIQA $\uparrow$} & \textbf{HS $\uparrow$} & \textbf{ARC-E $\uparrow$} & \textbf{ARC-C $\uparrow$} & \textbf{WG $\uparrow$} & \textbf{RTE $\uparrow$} & \textbf{OQA $\uparrow$} & \textbf{BoolQ $\uparrow$} & \textbf{Avg $\uparrow$}  \\
\midrule
\texttt{OATS} & \multirow{5}{*}{3:8+64LR} & 58.88 & 40.76 & 67.35 & 63.71 & 39.48 & 42.68 & 24.32 & 53.91 & 52.71 & 28.40 & 63.98 & 46.15 \\
\texttt{Hassle-free-SparseGPT} &  & 29.32 & 21.46 & 32.06 & 68.66 & 51.99 & 50.97 & 30.38 & 63.85 & 53.07 & 32.00 & 71.31 & 52.78 \\
\texttt{Hassle-free-ALPS} &  & 23.93 & 18.20 & 26.31 & 70.62 & 56.54 & 54.42 & 30.12 & 64.72 & 55.23 & 32.80 & 71.96 & 54.55 \\
\cellcolor{darkgreen!10}\texttt{\ourmethod} &  & \cellcolor{darkgreen!10}23.07 & \cellcolor{darkgreen!10}18.03 & \cellcolor{darkgreen!10}24.84 & \cellcolor{darkgreen!10}71.06 & \cellcolor{darkgreen!10}56.96 & \cellcolor{darkgreen!10}57.70 & \cellcolor{darkgreen!10}32.59 & \cellcolor{darkgreen!10}66.69 & \cellcolor{darkgreen!10}54.51 & \cellcolor{darkgreen!10}33.00 & \cellcolor{darkgreen!10}66.70 & \cellcolor{darkgreen!10}54.90 \\
\cellcolor{darkgreen!10}\texttt{\ourmethod-TM} &  & \cellcolor{darkgreen!10}\textbf{18.11} & \cellcolor{darkgreen!10}\textbf{14.26} & \cellcolor{darkgreen!10}\textbf{20.47} & \cellcolor{darkgreen!10}\textbf{74.05} & \cellcolor{darkgreen!10}\textbf{61.85} & \cellcolor{darkgreen!10}\textbf{60.73} & \cellcolor{darkgreen!10}\textbf{34.73} & \cellcolor{darkgreen!10}\textbf{65.98} & \cellcolor{darkgreen!10}\textbf{54.51} & \cellcolor{darkgreen!10}\textbf{34.80} & \cellcolor{darkgreen!10}\textbf{76.91} & \cellcolor{darkgreen!10}\textbf{57.94} \\
\midrule
\texttt{OATS} & \multirow{5}{*}{4:8+64LR} & 16.38 & 10.88 & 17.23 & 75.84 & 67.60 & 67.09 & 41.21 & 70.88 & 60.29 & 38.20 & 73.61 & 61.84 \\
\texttt{Hassle-free-SparseGPT} &  & 14.65 & 9.88 & 15.21 & 77.09 & 69.95 & 69.32 & 41.81 & 71.27 & 56.32 & 40.60 & 79.39 & 63.22 \\
\texttt{Hassle-free-ALPS} &  & 14.04 & 9.44 & 14.45 & 76.82 & 71.19 & 71.04 & 44.45 & \textbf{72.77} & 56.68 & 40.20 & 78.13 & 63.91 \\
\cellcolor{darkgreen!10}\texttt{\ourmethod} &  & \cellcolor{darkgreen!10}13.74 & \cellcolor{darkgreen!10}9.21 & \cellcolor{darkgreen!10}14.24 & \cellcolor{darkgreen!10}76.88 & \cellcolor{darkgreen!10}72.05 & \cellcolor{darkgreen!10}70.16 & \cellcolor{darkgreen!10}44.80 & \cellcolor{darkgreen!10}72.14 & \cellcolor{darkgreen!10}61.01 & \cellcolor{darkgreen!10}41.40 & \cellcolor{darkgreen!10}80.89 & \cellcolor{darkgreen!10}64.92 \\
\cellcolor{darkgreen!10}\texttt{\ourmethod-TM} &  & \cellcolor{darkgreen!10}\textbf{13.02} & \cellcolor{darkgreen!10}\textbf{8.64} & \cellcolor{darkgreen!10}\textbf{13.70} & \cellcolor{darkgreen!10}\textbf{78.24} & \cellcolor{darkgreen!10}\textbf{72.59} & \cellcolor{darkgreen!10}\textbf{73.11} & \cellcolor{darkgreen!10}\textbf{47.35} & \cellcolor{darkgreen!10}71.98 & \cellcolor{darkgreen!10}\textbf{63.18} & \cellcolor{darkgreen!10}\textbf{42.40} & \cellcolor{darkgreen!10}\textbf{80.49} & \cellcolor{darkgreen!10}\textbf{66.17} \\
\midrule
\texttt{OATS} & \multirow{5}{*}{2:4+64LR} & 21.59 & 14.76 & 23.41 & 72.74 & 60.70 & 60.86 & 34.81 & 65.51 & 57.76 & 35.20 & 68.32 & 56.99 \\
\texttt{Hassle-free-SparseGPT} &  & 17.77 & 12.38 & 18.71 & 74.81 & 65.04 & 66.16 & 38.57 & 70.09 & 54.87 & 38.40 & 77.71 & 60.71 \\
\texttt{Hassle-free-ALPS} &  & 16.15 & 11.38 & 16.71 & 75.19 & 67.10 & 64.44 & 38.91 & 69.53 & 59.93 & 39.40 & 78.38 & 61.61 \\
\cellcolor{darkgreen!10}\texttt{\ourmethod} &  & \cellcolor{darkgreen!10}15.76 & \cellcolor{darkgreen!10}11.23 & \cellcolor{darkgreen!10}16.25 & \cellcolor{darkgreen!10}76.50 & \cellcolor{darkgreen!10}67.61 & \cellcolor{darkgreen!10}67.21 & \cellcolor{darkgreen!10}40.10 & \cellcolor{darkgreen!10}70.24 & \cellcolor{darkgreen!10}64.26 & \cellcolor{darkgreen!10}38.20 & \cellcolor{darkgreen!10}78.29 & \cellcolor{darkgreen!10}62.80 \\
\cellcolor{darkgreen!10}\texttt{\ourmethod-TM} &  & \cellcolor{darkgreen!10}\textbf{14.34} & \cellcolor{darkgreen!10}\textbf{9.78} & \cellcolor{darkgreen!10}\textbf{14.88} & \cellcolor{darkgreen!10}\textbf{77.48} & \cellcolor{darkgreen!10}\textbf{69.58} & \cellcolor{darkgreen!10}\textbf{67.21} & \cellcolor{darkgreen!10}\textbf{40.53} & \cellcolor{darkgreen!10}\textbf{71.27} & \cellcolor{darkgreen!10}\textbf{61.37} & \cellcolor{darkgreen!10}\textbf{39.80} & \cellcolor{darkgreen!10}\textbf{79.51} & \cellcolor{darkgreen!10}\textbf{63.34} \\
\midrule
\midrule[1pt]
\rowcolor{NavyBlue!10}\textbf{{Meta-Llama-3-8B Dense}} & -- & 9.44 & 6.14 & 11.18 & 80.79 & 79.17 & 77.69 & 53.33 & 72.85 & 69.68 & 45.00 & 81.44 & 69.99 \\
\bottomrule
\end{tabular}%
}
\end{table}

\begin{table}[htbp]
\centering
\begin{minipage}[c]{0.63\textwidth}
    \resizebox{\textwidth}{!}{%
    \begin{tabular}{l|c|ccc|ccc}
    \toprule
    \textbf{Method} & \textbf{Config} & \textbf{C4 $\downarrow$} & \textbf{WT2 $\downarrow$} & \textbf{PTB $\downarrow$} & \textbf{PIQA $\uparrow$} & \textbf{ARC-E $\uparrow$} & \textbf{ARC-C $\uparrow$} \\
    \midrule
    Wanda &  & 38.21 & 26.89 & 47.13 & 67.63 & 49.37 & 29.01 \\
    \cellcolor{darkred!10}Wanda-TM &  & \cellcolor{darkred!10}15.91 & \cellcolor{darkred!10}11.03 & \cellcolor{darkred!10}17.60 & \cellcolor{darkred!10}75.14 & \cellcolor{darkred!10}63.97 & \cellcolor{darkred!10}40.19 \\
    SparseGPT &  & 22.65 & 16.22 & 25.15 & 71.16 & 56.48 & 32.59 \\
    \cellcolor{darkred!10}SparseGPT-TM &  & \cellcolor{darkred!10}15.30 & \cellcolor{darkred!10}10.83 & \cellcolor{darkred!10}16.77 & \cellcolor{darkred!10}76.28 & \cellcolor{darkred!10}\underline{65.28} & \cellcolor{darkred!10}\underline{40.36} \\
    ALPS &  & 19.62 & 14.50 & 21.73 & 73.78 & 60.06 & 35.84 \\
    \cellcolor{darkred!10}ALPS-TM & \multirow{-6}{*}{2:4} & \cellcolor{darkred!10}\underline{14.96} & \cellcolor{darkred!10}\underline{10.65} & \cellcolor{darkred!10}\underline{16.35} & \cellcolor{darkred!10}\underline{76.88} & \cellcolor{darkred!10}65.03 & \cellcolor{darkred!10}39.85 \\
    \midrule
    Wanda &  & 22.70 & 15.58 & 26.62 & 72.03 & 58.63 & 36.09 \\
    \cellcolor{darkred!10}Wanda-TM &  & \cellcolor{darkred!10}13.99 & \cellcolor{darkred!10}9.28 & \cellcolor{darkred!10}15.09 & \cellcolor{darkred!10}77.48 & \cellcolor{darkred!10}68.22 & \cellcolor{darkred!10}42.75 \\
    SparseGPT &  & 17.59 & 12.29 & 18.48 & 75.68 & 63.38 & 39.59 \\
    \cellcolor{darkred!10}SparseGPT-TM &  & \cellcolor{darkred!10}13.68 & \cellcolor{darkred!10}9.28 & \cellcolor{darkred!10}14.51 & \cellcolor{darkred!10}\underline{78.07} & \cellcolor{darkred!10}\underline{70.29} & \cellcolor{darkred!10}\underline{43.94} \\
    ALPS &  & 16.06 & 11.17 & 16.60 & 76.12 & 66.25 & 40.87 \\
    \cellcolor{darkred!10}ALPS-TM & \multirow{-6}{*}{4:8} & \cellcolor{darkred!10}\underline{13.59} & \cellcolor{darkred!10}\underline{9.15} & \cellcolor{darkred!10}\underline{14.18} & \cellcolor{darkred!10}77.58 & \cellcolor{darkred!10}69.57 & \cellcolor{darkred!10}\underline{43.94} \\
    \midrule
    OATS &  & 21.03 & 14.54 & 24.15 & 73.67 & 59.68 & 37.12 \\
    Hf-SparseGPT &  & 20.05 & 15.03 & 22.01 & 74.05 & 60.52 & 36.18 \\
    Hf-ALPS &  & 17.89 & 13.07 & 19.11 & 74.54 & 65.53 & 39.08 \\
    \cellcolor{darkgreen!10}3BASiL & & \cellcolor{darkgreen!10}15.20 & \cellcolor{darkgreen!10}10.64 & \cellcolor{darkgreen!10}15.80 & \cellcolor{darkgreen!10}76.71 & \cellcolor{darkgreen!10}70.08 & \cellcolor{darkgreen!10}43.52 \\
    \cellcolor{darkgreen!10}3BASiL-TM & \multirow{-5}{*}{2:8+LR} & \cellcolor{darkgreen!10}13.81 & \cellcolor{darkgreen!10}9.50 & \cellcolor{darkgreen!10}14.74 & \cellcolor{darkgreen!10}77.15 & \cellcolor{darkgreen!10}73.36 & \cellcolor{darkgreen!10}44.54 \\
    \midrule
    OATS &  & 16.87 & 11.43 & 18.53 & 75.24 & 65.91 & 39.85 \\
    Hf-SparseGPT &  & 16.16 & 11.36 & 16.71 & 75.79 & 67.55 & 41.04 \\
    Hf-ALPS &  & 14.85 & 10.20 & 15.42 & 77.15 & 69.40 & 43.64 \\
    \cellcolor{darkgreen!10}3BASiL & & \cellcolor{darkgreen!10}13.73 & \cellcolor{darkgreen!10}9.29 & \cellcolor{darkgreen!10}14.62 & \cellcolor{darkgreen!10}\textbf{78.45} & \cellcolor{darkgreen!10}71.42 & \cellcolor{darkgreen!10}43.43 \\
    \cellcolor{darkgreen!10}3BASiL-TM & \multirow{-5}{*}{3:8+LR} & \cellcolor{darkgreen!10}\textbf{13.01} & \cellcolor{darkgreen!10}\textbf{8.69} & \cellcolor{darkgreen!10}\textbf{13.74} & \cellcolor{darkgreen!10}77.80 & \cellcolor{darkgreen!10}\textbf{75.00} & \cellcolor{darkgreen!10}\textbf{47.44} \\
    \midrule
    \rowcolor{NavyBlue!10} \textbf{Llama3-8B} Dense& -- & 9.44 & 6.14 & 11.18 & 80.79 & 77.69 & 53.33 \\
    \bottomrule
    \end{tabular}%
    }
\end{minipage}%
\hfill
\begin{minipage}[c]{0.35\textwidth}
\caption{One-shot (N:M Sparse + LR) decomposition performance of Llama3-8B model. The compression ratio (percentage of non-zero parameters retained) is fixed to be $\rho=0.5$. For Perplexity, $(\downarrow)$ lower values are preferred. For zero-shot tasks, $(\uparrow)$ higher values are preferred. Bolded values correspond to the overall best compression scheme that satisfies $\rho=0.5$. Underlined values correspond to the best pure pruning algorithm for the same compression. This shows the universality of \texttt{transformer-matching} to pure sparsity constraints.}
\label{tab:compression-rho}
\end{minipage}
\end{table}

\begin{table}[htbp]
\centering
\begin{minipage}[c]{0.49\textwidth}
\resizebox{0.99\textwidth}{!}{%
\begin{tabular}{l|ccc|c}
\toprule
\textbf{Method} & \textbf{C4 $\downarrow$} & \textbf{WT2 $\downarrow$} & \textbf{PTB $\downarrow$} & \textbf{Time (hrs) $\downarrow$} \\
\midrule
\texttt{OATS-10} & 11.75 & 10.48 & 14.65 & 5.81 \\
\texttt{Hf-SparseGPT} & 11.58 & 10.17 & 14.39 & 5.97 \\
\texttt{Hf-ALPS-10} & 11.56 & 10.05 & 14.33 & 4.33 \\
\cellcolor{darkgreen!10}\texttt{3BASiL} & \cellcolor{darkgreen!10}\textbf{11.53} & \cellcolor{darkgreen!10}\textbf{10.04} & \cellcolor{darkgreen!10}\textbf{14.26} & \cellcolor{darkgreen!10}\textbf{4.20} \\
\midrule
\rowcolor{NavyBlue!10}\textbf{Dense} & 11.44 & 9.56 & 14.04 & -- \\
\bottomrule
\end{tabular}%
}
\end{minipage}
\hfill
\begin{minipage}[c]{0.49\textwidth}
\caption{One-shot (2:4 + 112) decomposition of OPT-30B model. This configuration results in efficient inference. We limit the compression runtime to 6 A100 GPU hours. \texttt{3BASiL-TM} largely exceeds this period. We limit the alternating minimization steps of \texttt{Hf-ALPS} and \texttt{OATS} to 10 to fit within the time constraint.}
\label{tab:opt30b-comparison}
\end{minipage}
\end{table}

\begin{table}[htbp]
\vspace{-4mm}
\centering
\caption{One-shot $(50\% + 128)$ decomposition for Llama3.2-1B and Meta-Llama-3-8B models.}\label{tab:unstr_results}
\resizebox{\textwidth}{!}{%
\begin{tabular}{l|c|ccc|ccccccccc}
\toprule
\textbf{Method} & \textbf{Config} & \textbf{C4 $\downarrow$} & \textbf{WT2 $\downarrow$} & \textbf{PTB $\downarrow$} & \textbf{PIQA $\uparrow$} & \textbf{HS $\uparrow$} & \textbf{ARC-E $\uparrow$} & \textbf{ARC-C $\uparrow$} & \textbf{WG $\uparrow$} & \textbf{RTE $\uparrow$} & \textbf{OQA $\uparrow$} & \textbf{BoolQ $\uparrow$} & \textbf{Avg $\uparrow$} \\
\midrule
OATS &  & 17.99 & 12.16 & 21.40 & 71.71 & 57.96 & 57.28 & 33.79 & 59.98 & 52.71 & 33.00 & 63.94 & 53.80 \\
Hf-SparseGPT &  & 17.25 & 11.99 & 20.87 & 72.91 & 59.04 & 56.82 & 33.11 & 58.88 & 57.76 & 35.00 & 57.03 & 53.82 \\
Hf-ALPS &  & 16.81 & 11.66 & 20.12 & 72.80 & 59.92 & 57.62 & 33.11 & 58.64 & 55.96 & 35.20 & 59.66 & 54.11 \\
\cellcolor{darkgreen!10}3BASiL &  & \cellcolor{darkgreen!10}16.17 & \cellcolor{darkgreen!10}11.16 & \cellcolor{darkgreen!10}20.00 & \cellcolor{darkgreen!10}\textbf{73.83} & \cellcolor{darkgreen!10}60.42 & \cellcolor{darkgreen!10}58.04 & \cellcolor{darkgreen!10}34.47 & \cellcolor{darkgreen!10}60.38 & \cellcolor{darkgreen!10}53.79 & \cellcolor{darkgreen!10}\textbf{36.80} & \cellcolor{darkgreen!10}58.20 & \cellcolor{darkgreen!10}54.49 \\
\cellcolor{darkgreen!10}3BASiL-TM & \multirow{-5}{*}{50\%+128} & \cellcolor{darkgreen!10}\textbf{15.78} & \cellcolor{darkgreen!10}\textbf{10.87} & \cellcolor{darkgreen!10}\textbf{19.33} & \cellcolor{darkgreen!10}73.23 & \cellcolor{darkgreen!10}\textbf{60.66} & \cellcolor{darkgreen!10}\textbf{59.26} & \cellcolor{darkgreen!10}\textbf{34.56} & \cellcolor{darkgreen!10}\textbf{61.01} & \cellcolor{darkgreen!10}\textbf{59.21} & \cellcolor{darkgreen!10}36.60 & \cellcolor{darkgreen!10}\textbf{64.13} & \cellcolor{darkgreen!10}\textbf{56.08} \\
\midrule
\rowcolor{NavyBlue!10} \textbf{Llama-3.2-1B Dense} & -- & 14.01 & 9.75 & 17.59 & 74.59 & 63.66 & 60.48 & 36.26 & 60.69 & 56.68 & 37.20 & 63.98 & 56.69 \\
\midrule[1pt]
OATS &  & 12.25 & 7.78 & 12.92 & 78.40 & 75.32 & 73.99 & 49.15 & \textbf{73.80} & 58.84 & 41.80 & 79.42 & 66.34 \\
Hf-SparseGPT &  & 11.98 & 7.77 & 12.85 & 79.11 & 75.88 & 75.00 & 49.40 & 73.32 & 63.18 & 43.80 & 78.32 & 67.25 \\
Hf-ALPS &  &  12.09 & 7.99 & 12.86 & 78.78 & 76.29 & \textbf{76.52} & \textbf{51.19} & 73.09 & 60.65 & 40.20 & \textbf{81.62} & 67.29\\
\cellcolor{darkgreen!10}3BASiL &  & \cellcolor{darkgreen!10}11.51 & \cellcolor{darkgreen!10}7.47 & \cellcolor{darkgreen!10}12.36 & \cellcolor{darkgreen!10}79.54 & \cellcolor{darkgreen!10}\textbf{76.69} & \cellcolor{darkgreen!10}74.75 & \cellcolor{darkgreen!10}48.72 & \cellcolor{darkgreen!10}72.69 & \cellcolor{darkgreen!10}67.87 & \cellcolor{darkgreen!10}43.00 & \cellcolor{darkgreen!10}80.24 & \cellcolor{darkgreen!10}67.94 \\
\cellcolor{darkgreen!10}3BASiL-TM & \multirow{-5}{*}{50\%+128} & \cellcolor{darkgreen!10}\textbf{11.27} & \cellcolor{darkgreen!10}\textbf{7.30} & \cellcolor{darkgreen!10}\textbf{12.26} & \cellcolor{darkgreen!10}\textbf{79.65} & \cellcolor{darkgreen!10}76.07 & \cellcolor{darkgreen!10}75.84 & \cellcolor{darkgreen!10}47.78 & \cellcolor{darkgreen!10}71.98 & \cellcolor{darkgreen!10}\textbf{70.40} & \cellcolor{darkgreen!10}\textbf{44.20} & \cellcolor{darkgreen!10}80.70 & \cellcolor{darkgreen!10}\textbf{68.33} \\
\midrule
\rowcolor{NavyBlue!10} \textbf{Meta-Llama-3-8B Dense} & -- & 9.44 & 6.14 & 11.18 & 80.79 & 79.17 & 77.69 & 53.33 & 72.85 & 69.68 & 45.00 & 81.44 & 69.99 \\
\bottomrule
\end{tabular}%
}
\end{table}

\noindent\textbf{LoRA fine-tuning after one-shot compression}~~~
After applying \slr decomposition, the resulting low-rank components can serve as initialization for LoRA fine-tuning on downstream tasks to recover lost performance. We conducted limited LoRA fine-tuning on 10\% of the first C4 training dataset shard (approximately 15 million tokens), with detailed hyperparameters in \cref{sect:appendix-exp-details}. \cref{fig:c4ppl-lft} demonstrates that \texttt{LFT-3BASiL-TM} significantly reduces the C4 perplexity of \slr decompositions, particularly under aggressive compression regimes like 2:8+64LR. Moreover, while LoRA fine-tuning can recover a large portion of the performance lost due to compression, an advanced one-shot decomposition approach retains its advantage post fine-tuning. For instance,  \texttt{LFT-3BASiL-TM} still outperforms competing decomposition methods after LoRA fine-tuning of 2:8+64LR configurations, achieving approximately 8\% lower perplexity.

\begin{figure}[h]
    \centering
    \begin{subfigure}[b]{0.43\textwidth}
        \centering
        \includegraphics[width=\linewidth]{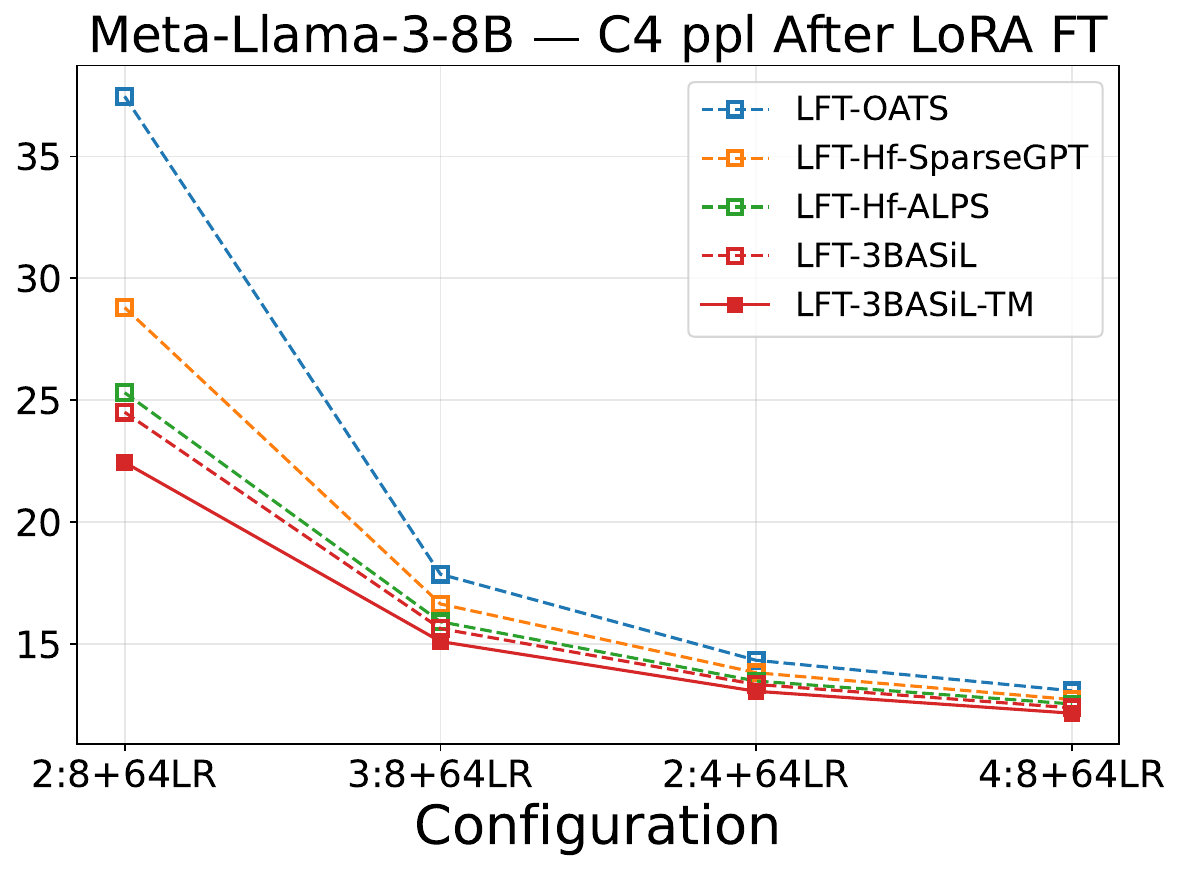}
        \vspace{0.2em}
        \caption{C4 ppl of Llama3-8B model under different \slr configurations after LoRA.}
        % \vspace{-1.0em}
        \label{fig:c4ppl-lft}
    \end{subfigure}
    \hfill
    \begin{subfigure}[b]{0.55\textwidth}
        \centering
        \includegraphics[width=\linewidth]{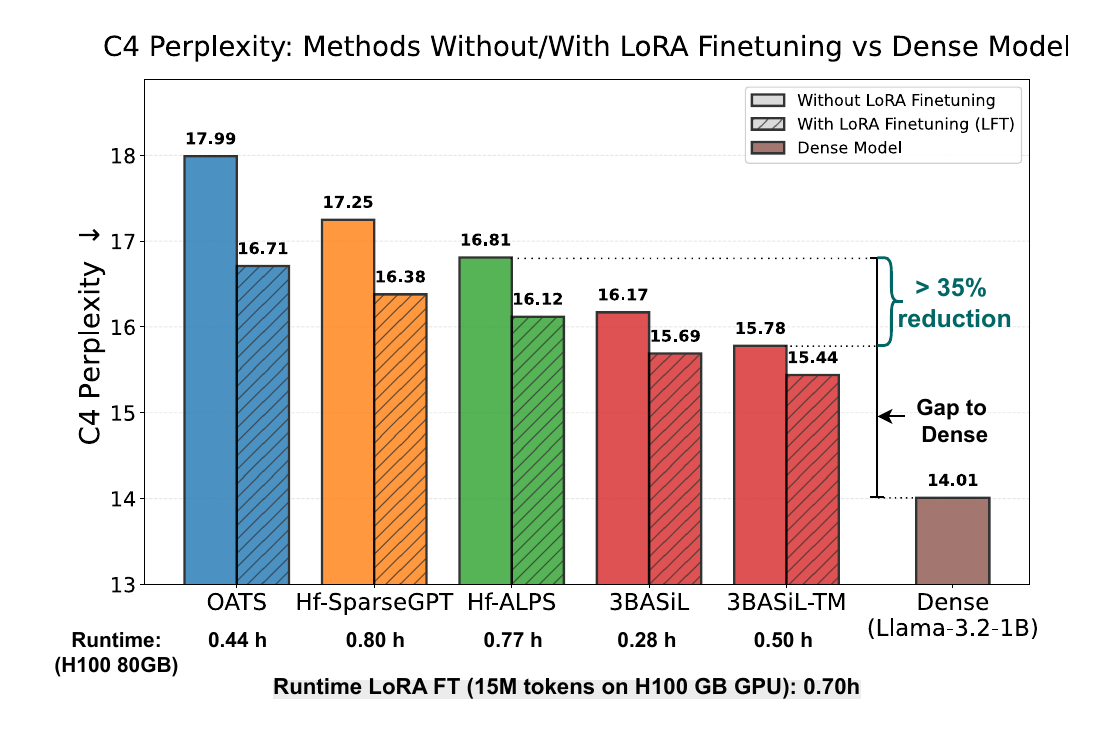}
        \caption{C4 perplexity gap to dense model (Llama3.2-1B) under (50\%+128LR) configuration.}
        \label{fig:nearlossless-1b}
    \end{subfigure}
    \caption{C4 perplexity performance of Llama3-8B \& Llama3.2-1B before/after LoRA fine-tuning.}
    \label{fig:combined}
\end{figure}

\vspace{-0.5em}
\section{Related Work}
\label{sec:related-work}
\vspace{-0.5em}
% \noindent \textbf{Robust PCA and Sparse plus Low-Rank Decomposition.}
% The decomposition of matrices into sparse and low-rank components has foundational roots in the \textit{Robust Principal Component Analysis} (RPCA) literature~\cite{candes2011robust, zhou2010stable}. In RPCA, one seeks to recover a matrix $\mathbf{M}$ as $\mathbf{M} = \mathbf{L} + \mathbf{S}$, where $\mathbf{L}$ is low-rank and $\mathbf{S}$ is sparse. Theoretical guarantees on exact recovery under incoherence and sparsity assumptions have made RPCA influential in statistics, signal processing, and machine learning.
\noindent \textbf{One-shot Sparse/Quantized plus Low-Rank compression}~~~ The seminal works of \cite{yu2017compressing} proposed a compression technique for a neural network using sparsity plus low-rank constraints. However, the authors study small-scale vision models. In addition, they consider a compression that needs to be repeated over multiple rounds (decomposing selected layers and followed by a retraining process). Our focus is different; we are interested in compressing at LLM-scale in one-shot (no expensive retraining).
Recent methods in LLM compression have focused on effectively combining low-rank decomposition with quantization or sparsity. EoRA \citep{DBLP:journals/corr/abs-2410-21271} has been proposed as a method to compensate for the loss produced by a general-purpose compressed weight $\mathcal{C}(\bfW)$ using a low-rank component, it does the low-rank fitting step once post the initial weight compression, which could include combinations of sparsity and quantization. LoftQ~\citep{li2023loftq} jointly optimizes quantization and LoRA initialization by solving $\min_{\mathbf{Q}, \mathbf{L}}\|\mathbf{W} - (\mathbf{Q} + \mathbf{L})\|_F$, where $\mathbf{W}$ represents the original weights, $\mathbf{Q}$ the quantized component, and $\mathbf{L}$ the low-rank component. LQ-LoRA~\citep{guo2023lq} extends this by incorporating Fisher information weighting, approximately solving $\min_{\mathbf{Q}, \mathbf{L}}\|\mathbf{F} \odot (\mathbf{W} - (\mathbf{Q} + \mathbf{L}))\|_F$. CALDERA~\citep{saha2024compressing} further considers the layer-wise reconstruction error, optimizing $\min_{\mathbf{Q}, \mathbf{L}}\|\mathbf{X}\mathbf{W} - \mathbf{X}(\mathbf{Q} + \mathbf{L})\|_F$ to maintain the outputs of individual layers rather than mere weight approximation.
From the \slr perspective, OATS~\citep{zhang2024oats} proposes an outlier-aware alternating minimization, effectively reducing to solving $\min_{\mathbf{S}, \mathbf{L}}\|\mathbf{D}\mathbf{W} - \mathbf{D}(\mathbf{S} + \mathbf{L})\|_F$ with $\mathbf{D} = \text{diag}(\mathbf{X}^T\mathbf{X})$, as noted by \cite{makni2025hasslefree}. HASSLE-free~\citep{makni2025hasslefree} directly tackles layer-wise reconstruction error $\min_{\mathbf{S}, \mathbf{L}}\|\mathbf{X}\mathbf{W} - \mathbf{X}(\mathbf{S} + \mathbf{L})\|_F$ using alternating minimization. 
While methods such as OATS and HASSLE-free separately optimize sparse and low-rank components, our proposed approach, \texttt{3BASiL}, distinctly utilizes a unified optimization framework via a 3-block ADMM formulation, jointly optimizing sparse and low-rank components simultaneously.

\noindent \textbf{Sparse plus Low-Rank structures in transformers}~~~
Beyond \textit{model compression}, sparse plus low-rank structures have a strong presence in the context of LLMs. LoRAPrune~\citep{zhang2024loraprune} is a purse sparsification method, which prunes a model (iteratively) by designing a memory-efficient LoRA-guided (low-rank structure) pruning criterion. In contrast, LoSA (low-rank Sparse Adaptation) \citep{huangdynamic} jointly applies LoRA fine-tuning and pruning in a unified framework to obtain a fine-tuned sparse-only (as opposed to \slr) model, by dynamically sparsifying the LoRA weights and adjusting their rank. SLTrain \citep{hansltrain} addresses \slr from a training perspective. It pre-trains an LLM using a fixed random sparse mask plus trainable low-rank factors (similar to LoRA), achieving comparable accuracy to dense training with far fewer parameters. SLTrain demonstrates the benefits of \slr structure for pre-training but it doesn't solve the post-hoc decomposition problem of a dense model. There are connections between our \textit{transformer-matching} step and SLTrain as they both train sparse (fixed support) and low-rank components, but they minimize different loss functions and serve different purposes. 

\noindent \textbf{ADMM approaches to compress networks}~~~ 
The Alternating Direction Method of Multipliers (ADMM) \citep{boyd2011distributed, davis2016convergence} is an effective optimization technique for problems with coupled variables that has been successfully applied to neural network compression. \cite{ye2018progressive} introduced ADMM-based progressive weight pruning that optimizes the original loss function under sparsity constraints, which \cite{ye2019adversarial} extended to preserve adversarial robustness during compression. In contrast, recent methods have scaled ADMM to LLMs through layer-wise reconstruction: \cite{bovza2024fast} employed ADMM to solve a convex problem recovering optimal weights on a fixed support of the weight matrix, while \cite{meng2024alps} utilized ADMM for a non-convex problem that jointly optimizes both support and weights. Our proposed method differs from these prior works as we explore a 3-block ADMM in model compression that simultaneously optimizes \slr components with theoretical convergence guarantees.

\noindent \textbf{Exact Low-Rank updates for layer-wise compression}~~~
The problem of exact low-rank updates found in \cref{eq:L-update} has original roots from classical reduced-rank regression methods \citep{izenman1975reduced,reinsel1998multivariate}, which provide closed-form solutions for optimally approximating linear regression models under rank constraints. Recent work, including CALDERA \citep{saha2024compressing} and the low-rank correction method by \cite{scetbon2024low}, applies these closed-form updates to compress large language models into $\mathbf{W} \approx \mathbf{Q}+\mathbf{LR}$. 
%In particular, \cite{saha2024compressing} derive exact updates for the more challenging optimization problem $\min_{\mathbf{M}} \|\mathbf{Y} - \mathbf{XM}\|_F,\, s.t.\, \text{rank}(\mathbf{M}) \leq r.$ 
%%%Our method builds on 
We also use these exact low-rank updates by integrating them directly in \cref{eq:L-update} within our ADMM framework for \slr decomposition.

\vspace{-0.5em}
\section{Conclusion and limitations}
\label{sec:conclusion}
\vspace{-0.5em}
% !TEX root = ../main.tex

We present \ourmethod as a highly-efficient \slr decomposition algorithm with theoretical convergence guarantees. It provides high-quality solutions to the layer-wise decomposition problem presented in \cref{eq:original} in terms of objective minimization (\cref{fig:attn-layers-loss} and \cref{fig:mlp-layers-loss}) compared to competing \slr decomposition methods. We further refine these decomposed weights with our novel (memory-efficient) transformer matching step \texttt{TM} that can enhance any \slr decomposition. This shows that one route for optimal compression results (in the context of $\mathcal{C}(\mathbf{W})+\mathbf{LR}$) is to unfold the LLM compression into 3 minimization steps: (i) [layer-wise reconstruction] this is the loss that has been considered in many SOTA pruning/quantization algorithms \citep{frantar2023sparsegpt, meng2024alps, saha2024compressing,frantar2022gptq,
meng2024osscar}, (ii) [transformer-matching] this is an intermediate loss function (to be optimized in a memory-efficient manner) which is a more reliable approximation to the true loss function than simple layer-wise reconstruction, and (iii) [LoRA fine-tuning] plugs the obtained low-rank components as \textit{smart initialization} for LoRA to minimize the true LLM loss function. We believe that our 3-block ADMM approach and \texttt{TM} can generalize to quantization or quantized-sparse constraints. We leave these explorations for future works.While we have shown how to integrate sparsity allocation mechanisms like OWL to our framework, it remains to explore dedicated methods that can algorithmically allocate different sparsity/rank configurations to different layers to further improve efficiency-utility-computations tradeoffs.

\subsection*{Acknowledgements}
This research is supported in part by grants from the Office of Naval 
Research (N000142512504, N000142212665). We acknowledge the MIT Engaging cluster for providing HPC resources that have contributed to the research results reported within this paper. Additionally, we thank Google for providing us with Google Cloud Credits. 
We thank Shibal Ibrahim, Ryan Lucas, and Gabriel Afriat for their helpful discussions.

%%%%%%%%%%%%%%%%%%%%%%%%%%%%%%%%%%%%%%%%%%%%%%%%%%%%%%%%%%%%
\newpage
\bibliographystyle{plainnat}
\bibliography{ref}
%%%%%%%%%%%%%%%%%%%%%%%%%%%%%%%%%%%%%%%%%%%%%%%%%%%%%%%%%%%%
\newpage 
\appendix
\section*{Appendix}\label{sec:appendix}
\section{Proofs of Theorem \ref{thm:admm}}\label{sect:proofadmm}
% % !TEX root =../main.tex

\begin{proof} For conciseness, throughout the proof, we denote $\mathbf{H}=\mathbf{X}^\top\mathbf{X}+\lambda\mathbf{I}$ and $\mathbf{G} = \left(\mathbf{X}^\top\mathbf{X}+\lambda\mathbf{I}\right)\widehat{\mathbf{W}}$. We denote $C_F$ as a large constant such that
\begin{equation}\label{eq:constant}
\max\{1,\| \mathbf{H}^{-1/2} \|_2,\,\|\mathbf{H}^{1/2}\|_2,\,\|\mathbf{H}\|_2,\,\|\widehat{\mathbf{W}}\|_F,\,\|\mathbf{G}\|_F\}\le C_F.
\end{equation}
For $t\ge1$, denote
\[
a_t:=\|\mathbf{D}^{(t)}\|_F+\frac{\|\mathbf{V}^{(t)}\|_F}{\rho_{t-1}}
+\frac{\|\mathbf{V}^{(t-1)}\|_F}{\rho_{t-1}},
\qquad b_t:=\max\left\{a_t,\frac{\|\mathbf{L}^{(t)}\|_F}{C_F^3}-1\right\}.
\]
The additional term in $b_t$ accounts for retaining the preceding low-rank iterate when a candidate is rejected. Thus, $a_t\le b_t$ and $\|\mathbf{L}^{(t)}\|_F\le C_F^3(1+b_t)$ at every iteration.
To establish the theorem, we first present the following three lemmas.

\begin{lemma}\label{lemma0}
Let $\left\{\mathbf{D}^{(t)}\right\}_{t=0}^\infty$, $\left\{\mathbf{L}^{(t)}\right\}_{t=0}^\infty$ and $\left\{\mathbf{V}^{(t)}\right\}_{t=0}^\infty$ be the sequence generated according to update rule \eqref{eq:finalupdate}. Then for any $t \ge 1$ at which the low-rank candidate is accepted, it holds
\begin{equation}
	\|\mathbf{L}^{(t)}\|_F \le C_F^3\left(1+\|\mathbf{D}^{(t)}\|_F + \frac{\|\mathbf{V}^{(t)}\|_F}{\rho_{t-1}} + \frac{\|\mathbf{V}^{(t-1)}\|_F}{\rho_{t-1}}\right).
\end{equation}
If the candidate is rejected, $\mathbf{L}^{(t)}=\mathbf{L}^{(t-1)}$.
\end{lemma}

\begin{lemma}\label{lemma1}
Let $\left\{\mathbf{D}^{(t)}\right\}_{t=0}^\infty$ and $\left\{\mathbf{V}^{(t)}\right\}_{t=0}^\infty$ be the sequence generated according to update rule \eqref{eq:finalupdate}. Then for any $t \ge 1$, it holds
\begin{equation}
	{\| \mathbf{V}^{(t+1)} \|_F } \le
	(C_F+C_F^4)\left(1+b_t\right).
\end{equation}
and
\begin{equation}
\| \mathbf{D}^{(t+1)} - \mathbf{D}^{(t)} \|_F \le
		 \frac{2C_F+2C_F^4}{\rho_{t}}\left(1+b_t\right).
\end{equation}
\end{lemma}

\begin{lemma}\label{lemma2}
Let $\left\{\mathbf{D}^{(t)}\right\}_{t=0}^\infty$ and $\left\{\mathbf{V}^{(t)}\right\}_{t=0}^\infty$ be the sequence generated according to update rule \eqref{eq:finalupdate}. Then for any $t \ge1$, it holds
\begin{equation}
\begin{aligned}
& b_t\le \exp\left(3(C_F+C_F^4)\sum_{s=1}^{t-1} \frac{1}{ \rho_{s-1} } \right) \cdot
		\left(  b_1  + 3(C_F+C_F^4)\sum_{s=1}^{t-1} \frac{1}{ \rho_{s-1} }   \right).
\end{aligned}
\end{equation}
\end{lemma}

Returning to the proof of the main theorem, define
\begin{equation}
\begin{aligned}
\Tilde{C_{A}} &= 2(C_F+C_F^4)\Bigg[1+ \exp\left(3(C_F+C_F^4)\sum_{s=1}^{\infty} \frac{1}{ \rho_{s-1} } \right) \cdot \left(  b_1  + 3(C_F+C_F^4) \sum_{s=1}^{\infty} \frac{1}{ \rho_{s-1} }   \right)\Bigg]\\
C_A &= \max\{\Tilde{C_{A}},2\|\mathbf{V}^{(0)}\|_F,2\|\mathbf{V}^{(1)}\|_F\}.
\end{aligned}
\end{equation}
Here $C_{A}$ is a constant depending on $\mathbf{X}$, $\widehat{\mathbf{W}}$, $\lambda$, $\rho_0$, the initialization, and $\sum_{t=0}^\infty 1/\rho_t$.

Lemma \ref{lemma1} together with Lemma \ref{lemma2} yields
\begin{equation}
\| \mathbf{D}^{(t+1)} - \mathbf{D}^{(t)} \|_F \le
		 \frac{2C_F+2C_F^4}{\rho_{t}}\left(1+b_t\right)\le \frac{C_A}{\rho_t}.
\end{equation}
and
\begin{equation}
\| \mathbf{V}^{(t+1)} \|_F \le
		 (C_F+C_F^4)\left(1+b_t\right)\le \frac{C_A}{2}.
\end{equation}
It then follows from $\mathbf{V}$-update rule and triangle inequality that
\begin{equation}
\begin{aligned}
\| \mathbf{S}^{(t+1)} - \mathbf{S}^{(t)} \|_F & \le \| \mathbf{S}^{(t+1)} - \mathbf{D}^{(t+1)} \|_F + \| \mathbf{D}^{(t+1)} - \mathbf{D}^{(t)} \|_F + \| \mathbf{S}^{(t)} - \mathbf{D}^{(t)} \|_F\\
& \le \frac{\|\mathbf{V}^{(t+1)}\|_F+\|\mathbf{V}^{(t)}\|_F}{\rho_t} + \| \mathbf{D}^{(t+1)} - \mathbf{D}^{(t)} \|_F + \frac{\|\mathbf{V}^{(t)}\|_F+\|\mathbf{V}^{(t-1)}\|_F}{\rho_{t-1}}\\
& \le \frac{3C_A}{\rho_{t-1}}.
\end{aligned}
\end{equation}
% Since $\mathbf{S}^{(t+1)} - \mathbf{D}^{(t+1)} = (\mathbf{V}^{(t+1)}-\mathbf{V}^{(t)})/\rho_t$ from the $\mathbf{V}$-update rule, we have
% \begin{equation}
% \| \mathbf{S}^{(t+1)} - \mathbf{S}^{(t)} \|_F \le \| \mathbf{S}^{(t+1)} - \mathbf{D}^{(t+1)} \|_F + \| \mathbf{D}^{(t+1)} - \mathbf{D}^{(t)} \|_F + \| \mathbf{S}^{(t)} - \mathbf{D}^{(t)} \|_F.
% \end{equation}

% Using the fact that $\mathbf{S}^{(t)} - \mathbf{D}^{(t)} = (\mathbf{V}^{(t)}-\mathbf{V}^{(t-1)})/\rho_{t-1}$, we get
% \begin{equation}
% \| \mathbf{S}^{(t+1)} - \mathbf{S}^{(t)} \|_F \le \frac{\|\mathbf{V}^{(t+1)}\|_F+\|\mathbf{V}^{(t)}\|_F}{\rho_t} + \| \mathbf{D}^{(t+1)} - \mathbf{D}^{(t)} \|_F + \frac{\|\mathbf{V}^{(t)}\|_F+\|\mathbf{V}^{(t-1)}\|_F}{\rho_{t-1}}.
% \end{equation}
% By Lemmas \ref{lemma1} and \ref{lemma2}, this is bounded by $C_A/\rho_t$ (after adjusting the constant if necessary).
According to the low-rank acceptance rule, we have, for both accepted and rejected candidates,
\begin{equation}
\begin{aligned}
\|\mathbf{L}^{(t+1)} - \mathbf{L}^{(t)}\|_F
&\le C_L \|\mathbf{S}^{(t+1)} - \mathbf{S}^{(t)}\|_F\\
&\le \frac{3C_L C_A}{\rho_{t-1}}.
\end{aligned}
\end{equation}
Therefore, for a large constant $C=3\max\{1,C_L\}C_A$, we have
\begin{equation}
\max\{\| \mathbf{S}^{(t+1)} - \mathbf{S}^{(t)} \|_F, \| \mathbf{L}^{(t+1)} - \mathbf{L}^{(t)} \|_F\} \le \frac{C}{\rho_{t-1}}.
\end{equation}
The increment bounds imply finite total movement: the distance between any two later iterates is bounded by a tail of the convergent series $\sum_t C/\rho_{t-1}$. Since $\sum_{s=0}^\infty 1/\rho_s <\infty$, both $\{\mathbf{S}^{(t)}\}_{t=0}^\infty$ and $\{\mathbf{L}^{(t)}\}_{t=0}^\infty$ are Cauchy sequences. Therefore, there exist matrices $\bar{\mathbf{S}}$ and $\bar{\mathbf{L}}$ such that $\mathbf{S}^{(t)} \rightarrow \bar{\mathbf{S}}$ and $\mathbf{L}^{(t)} \rightarrow \bar{\mathbf{L}}$ as $t \rightarrow \infty$. Setting $\bar{\mathbf{W}} = \bar{\mathbf{S}} + \bar{\mathbf{L}}$, we conclude that $\mathbf{S}^{(t)}+\mathbf{L}^{(t)} \rightarrow \bar{\mathbf{W}}$ as $t \rightarrow \infty$.
\end{proof}

\subsection{Proof of Lemma \ref{lemma0}}
\begin{proof}
For an accepted candidate, the $\mathbf{L}$-update rule in \eqref{eq:finalupdate}, together with \eqref{eq:constant} yields
\begin{equation}
\begin{aligned}
\|\mathbf{L}^{(t)} \|_F & = \left\|\mathbf{H}^{-1/2} P_r(\mathbf{H}^{1/2} (\widehat{\mathbf{W}}- \mathbf{S}^{(t)}))\right\|_F \\
& \le \| \mathbf{H}^{-1/2}\|_2 \left\|P_r(\mathbf{H}^{1/2} (\widehat{\mathbf{W}}- \mathbf{S}^{(t)}))\right\|_F \\
& \le C_F \left\|\mathbf{H}^{1/2} (\widehat{\mathbf{W}}- \mathbf{S}^{(t)})\right\|_F\\
& \le C_F\left\|\mathbf{H}^{1/2}\right\|_2 \|\widehat{\mathbf{W}}\|_F + C_F\left\|\mathbf{H}^{1/2}\right\|_2 \left\|\mathbf{S}^{(t)}\right\|_F\\
&\le C_F^2 \|\widehat{\mathbf{W}}\|_F +C_F^2 \|\mathbf{S}^{(t)}\|_F,
\end{aligned}
\end{equation}
where the second inequality follows from the truncated-SVD norm bound $\|P_r(\mathbf{A})\|_F\le\|\mathbf{A}\|_F$. It then follows from the $\mathbf{V}$-update rule in \eqref{eq:finalupdate} that
\begin{equation}
\begin{aligned}
\|\mathbf{L}^{(t)} \|_F &\le C_F^2 \|\widehat{\mathbf{W}}\|_F +C_F^2 \|\mathbf{S}^{(t)}\|_F\\
& = C_F^2 \|\widehat{\mathbf{W}}\|_F +C_F^2 \left\|\mathbf{D}^{(t)} + \frac{\mathbf{V}^{(t)}-\mathbf{V}^{(t-1)}}{\rho_{t-1}}\right\|_F \\
& \le C_F^3\left(1+\|\mathbf{D}^{(t)}\|_F + \frac{\|\mathbf{V}^{(t)}\|_F}{\rho_{t-1}} + \frac{\|\mathbf{V}^{(t-1)}\|_F}{\rho_{t-1}}\right).
\end{aligned}
\end{equation}
For a rejected candidate, the acceptance rule retains $\mathbf{L}^{(t-1)}$.
\end{proof}

\subsection{Proof of Lemma \ref{lemma1}}

\begin{proof}
According to the $\mathbf{S}$-update rule in \eqref{eq:finalupdate}, it holds
\begin{equation}
\begin{aligned}
\mathbf{S}^{(t+1)} - \mathbf{D}^{(t)} + \frac{\mathbf{V}^{(t)}}{\rho_t}
&	= (\mathbf{H}+\rho_t \mathbf{I})^{-1} (\mathbf{G} - \mathbf{H}\mathbf{L}^{(t)} - \mathbf{V}^{(t)} +\rho_t \mathbf{D}^{(t)}) - \mathbf{D}^{(t)} + \frac{\mathbf{V}^{(t)}}{\rho_t}	\\
&	= \left( (\mathbf{H}+\rho_t \mathbf{I} )^{-1} \rho_t -\mathbf{I} \right)	\mathbf{D}^{(t)} + (\mathbf{H}+\rho_t \mathbf{I} )^{-1} (\mathbf{G}- \mathbf{H}\mathbf{L}^{(t)}-\mathbf{V}^{(t)}) + \frac{\mathbf{V}^{(t)}}{\rho_t} \\
&	= -\frac{1}{\rho_t} \left( \mathbf{I} + \frac{\mathbf{H}}{\rho_t} \right)^{-1} \mathbf{H} \mathbf{D}^{(t)} +
			\frac{1}{\rho_t} \left( \mathbf{I} +  \frac{\mathbf{H}}{\rho_t} \right)^{-1} (\mathbf{G}- \mathbf{H}\mathbf{L}^{(t)}-\mathbf{V}^{(t)}) + \frac{\mathbf{V}^{(t)}}{\rho_t} \\
&	= \frac{1}{\rho_t} \left( \mathbf{I} + \frac{\mathbf{H}}{\rho_t} \right)^{-1} ( \mathbf{G}- \mathbf{H}\mathbf{L}^{(t)}-\mathbf{H}\mathbf{D}^{(t)} )
			+ \frac{1}{\rho_t}  \left[ \mathbf{I} - \left( \mathbf{I} +  \frac{\mathbf{H}}{\rho_t} \right)^{-1}  \right] \mathbf{V}^{(t)} \\
&	= \frac{1}{\rho_t} \left( \mathbf{I} + \frac{\mathbf{H}}{\rho_t} \right)^{-1} \left( \mathbf{G}- \mathbf{H}\mathbf{L}^{(t)}-\mathbf{H}\mathbf{D}^{(t)} + \frac{\mathbf{H}\mathbf{V}^{(t)}}{\rho_t} \right)
\end{aligned}
\end{equation}
Since $\mathbf{H}\succ0$ and $\rho_t>0$, every eigenvalue of $(\mathbf{I}+\mathbf{H}/\rho_t)^{-1}$ lies in $(0,1)$, so its spectral norm is at most one. Therefore, we obtain
\begin{equation}\label{eq:key1}
\begin{aligned}
\left\| \mathbf{S}^{(t+1)} - \mathbf{D}^{(t)} + \frac{\mathbf{V}^{(t)}}{\rho_t} \right\|_F
& \le \frac{1}{\rho_t} \left\| \left( \mathbf{I} + \frac{\mathbf{H}}{\rho_t} \right)^{-1} \right\|_2 \left\| \mathbf{G}- \mathbf{H}\mathbf{L}^{(t)}-\mathbf{H}\mathbf{D}^{(t)} + \frac{\mathbf{H}\mathbf{V}^{(t)}}{\rho_t} \right\|_F \\
& \le \frac{1}{\rho_t} \left\| \mathbf{G}- \mathbf{H}\mathbf{L}^{(t)}-\mathbf{H}\mathbf{D}^{(t)} + \frac{\mathbf{H}\mathbf{V}^{(t)}}{\rho_t} \right\|_F \\
& \le \frac{1}{\rho_t} \left( \| \mathbf{G}- \mathbf{H}\mathbf{L}^{(t)}-\mathbf{H}\mathbf{D}^{(t)} \|_F + \frac{ \| \mathbf{H}\mathbf{V}^{(t)}\|_F}{ \rho_t } \right).
\end{aligned}
\end{equation}

Since $\mathbf{D}^{(t)}\in\mathcal{S}$ is feasible for the exact $\mathbf{D}$-projection,
\begin{equation}\label{eq:key2}
\left\|\mathbf{D}^{(t+1)}-\mathbf{S}^{(t+1)}-\frac{\mathbf{V}^{(t)}}{\rho_t}\right\|_F
\le
\left\|\mathbf{D}^{(t)}-\mathbf{S}^{(t+1)}-\frac{\mathbf{V}^{(t)}}{\rho_t}\right\|_F.
\end{equation}
This inequality applies to both unstructured and N:M sparsity.

Together with \eqref{eq:key1}, we get
\begin{equation}\label{eq:key3}
	\left\| \mathbf{D}^{(t+1)} - \mathbf{S}^{(t+1)}- \frac{\mathbf{V}^{(t)}}{\rho_t}  \right\|_F \le
	\frac{1}{\rho_t} \left(  \| \mathbf{G}- \mathbf{H}\mathbf{L}^{(t)}-\mathbf{H}\mathbf{D}^{(t)}   \|_F + \frac{ \| \mathbf{H}\mathbf{V}^{(t)}\|_F}{ \rho_t }  \right).
\end{equation}
It then follows from the $\mathbf{V}$-update rule that
\begin{equation}\label{eq:key4}
	\frac{\|\mathbf{V}^{(t+1)}\|_F}{\rho_t}  = 	\left\| \mathbf{D}^{(t+1)} - \mathbf{S}^{(t+1)}- \frac{\mathbf{V}^{(t)}}{\rho_t}  \right\|_F \le
	\frac{1}{\rho_t} \left(  \| \mathbf{G}- \mathbf{H}\mathbf{L}^{(t)}-\mathbf{H}\mathbf{D}^{(t)}   \|_F + \frac{ \| \mathbf{H}\mathbf{V}^{(t)}\|_F}{ \rho_t }  \right)
\end{equation}
According to the definition of $b_t$ and the monotonicity of $\{\rho_t\}_{t=0}^\infty$, it holds
\begin{equation}\label{eq:key5}
\begin{aligned}
\| \mathbf{G}- \mathbf{H}\mathbf{L}^{(t)}-\mathbf{H}\mathbf{D}^{(t)} \|_F + \frac{ \| \mathbf{H}\mathbf{V}^{(t)}\|_F}{ \rho_t } &\le \|\mathbf{G}\|_F + \|\mathbf{H}\|_2\|\mathbf{L}^{(t)}\|_F + \|\mathbf{H}\|_2\|\mathbf{D}^{(t)}\|_F+\frac{\|\mathbf{H}\|_2\|\mathbf{V}^{(t)}\|_F}{\rho_t}\\
&\le C_F\left(1+\|\mathbf{D}^{(t)}\|_F+\frac{\|\mathbf{V}^{(t)}\|_F}{\rho_{t-1}}\right) + C_F \|\mathbf{L}^{(t)}\|_F \\
& \le (C_F + C_F^4)\left(1+b_t\right).
\end{aligned}
\end{equation}
Together with inequality \eqref{eq:key4}, this establishes the first inequality of the lemma. Furthermore, by summing up \eqref{eq:key1} and \eqref{eq:key3} and applying the triangle inequality, we verify the second inequality.
\end{proof}

\subsection{Proof of Lemma \ref{lemma2}}
\begin{proof}
It follows from Lemma~\ref{lemma1} that
\begin{equation}\label{eq:key6}
\begin{aligned}
\frac{\| \mathbf{V}^{(t+1)} \|_F }{\rho_t} & \le \frac{C_F + C_F^4}{\rho_t}\left(1+b_t\right)
\end{aligned}
\end{equation}
and
\begin{equation}
\begin{aligned}
\|\mathbf{D}^{(t+1)}\|_F & \le \|\mathbf{D}^{(t)}\|_F+ \| \mathbf{D}^{(t+1)} - \mathbf{D}^{(t)} \|_F \\
&\le \|\mathbf{D}^{(t)}\|_F + \frac{2C_F+2C_F^4}{\rho_{t}}\left(1+b_t\right).
\end{aligned}
\end{equation}
Summing up these two inequalities yields
\begin{equation}
\begin{aligned}
& \| \mathbf{D}^{(t+1)} \|_F + \frac{\| \mathbf{V}^{(t+1)} \|_F }{ \rho_{t} } + \frac{\| \mathbf{V}^{(t)} \|_F }{ \rho_{t} }
\le \| \mathbf{D}^{(t+1)} \|_F + \frac{\| \mathbf{V}^{(t+1)} \|_F }{ \rho_{t} } + \frac{\| \mathbf{V}^{(t)} \|_F }{ \rho_{t-1} } \\
& \qquad\le \frac{3C_F+3C_F^4}{\rho_{t}}\left(1+b_t\right) + \|\mathbf{D}^{(t)}\|_F + \frac{\| \mathbf{V}^{(t)} \|_F }{ \rho_{t-1} } \\
&\qquad \le \frac{3C_F+3C_F^4}{\rho_{t-1}}(1+b_t)+b_t
\quad \left(\rho_t\ge\rho_{t-1},\;\|\mathbf{D}^{(t)}\|_F+\frac{\|\mathbf{V}^{(t)}\|_F}{\rho_{t-1}}\le a_t\le b_t\right) \\
&\qquad = \left(1+ \frac{3C_F+3C_F^4}{\rho_{t-1}}\right)b_t + \frac{3C_F+3C_F^4}{\rho_{t-1}},
\end{aligned}
\end{equation}
By the definition of $a_t$, the above inequality can be rewritten as
\begin{equation}
a_{t+1} \le \left( 1 + \frac{3C_F+3C_F^4}{\rho_{t-1}}\right) b_t + \frac{3C_F+3C_F^4}{\rho_{t-1}}.
\end{equation}
If the low-rank candidate at iteration $t+1$ is accepted, Lemma~\ref{lemma0} gives $\|\mathbf{L}^{(t+1)}\|_F/C_F^3-1\le a_{t+1}$; otherwise, $\|\mathbf{L}^{(t+1)}\|_F/C_F^3-1=\|\mathbf{L}^{(t)}\|_F/C_F^3-1\le b_t$. Hence
\begin{equation}
b_{t+1} \le \max\{a_{t+1},b_t\}
\le \left( 1 + \frac{3C_F+3C_F^4}{\rho_{t-1}}\right) b_t + \frac{3C_F+3C_F^4}{\rho_{t-1}}.
\end{equation}
Therefore,
\begin{equation}
\begin{aligned}
	\frac{b_{t+1}}{ \prod_{s=1}^{t} ( 1+ 3(C_F+C_F^4)/\rho_{s-1} ) } \le \ &	\frac{b_{t}}{ \prod_{s=1}^{t-1} ( 1+ 3(C_F+C_F^4)/\rho_{s-1} ) } + \frac{3(C_F+C_F^4)}{\rho_{t-1} \prod_{s=1}^{t} ( 1+ 3(C_F+C_F^4)/\rho_{s-1} ) }		\\
	\le \ &
		\frac{b_{t}}{ \prod_{s=1}^{t-1} ( 1+ 3(C_F+C_F^4)/\rho_{s-1} ) } +  \frac{3(C_F+C_F^4)}{\rho_{t-1}}
\end{aligned}
\end{equation}
The second inequality uses the fact that every factor in the denominator product exceeds $1$, so the product is at least $1$.
It then follows from telescoping that
\begin{equation}
\frac{b_{t}}{ \prod_{s=1}^{t-1} ( 1+ 3(C_F+C_F^4)/\rho_{s-1} ) } \le b_1 + \sum_{s=1}^{t-1} \frac{3(C_F+C_F^4) }{ \rho_{s-1} }
\end{equation}
Using $1+x\le e^x$ with $x=3(C_F+C_F^4)/\rho_{s-1}\ge0$, we obtain
\begin{equation}
\prod_{s=1}^{t-1} ( 1+ 3(C_F+C_F^4)/\rho_{s-1} )\le \exp\left(3(C_F+C_F^4)\sum_{s=1}^{t-1}\frac{1}{\rho_{s-1}}\right),
\end{equation}
recalling the definition of $b_t$ completes the proof.
\end{proof}

\newpage
\section{Additional Experimental Details}\label{sect:appendix-exp-details}
\paragraph{Computing environments} All experiments were conducted on a computing cluster. Unless otherwise specified, we utilized an Intel Xeon Gold 6248 machine with 16 CPU cores and a single NVIDIA L40 48GB / A100 80GB / H100 80GB GPU. When runtime compression results are reported, all experiments have been run on the same node (including GPU) configuration. All language models and pruning methods were implemented using the PyTorch library \cite{paszke2017automatic}.

\paragraph{Implementation Details of \texttt{3BASiL}}
We use $\mathbf{H}^\prime = \mathbf{H}+0.005\text{diag}(\mathbf{X}^\top\mathbf{X})+0.005\operatorname{Tr}(\mathbf{X}^\top\mathbf{X})\mathbf{I}$.

With an iteration-dependent penalty $\rho_t$, the updates and low-rank candidate at iteration $t$ are as follows; the gated variant accepts or retains the low-rank component according to the rule following \cref{eq:L-update}:
\begin{alignat}{2}
&\mathbf{S}^{(t+1)}\!=\!\left( \mathbf{H}+\rho_t \mathbf{I}\right)^{-1} ( \mathbf{H}(\widehat{\mathbf{W}}-\mathbf{L}^{(t)}) - \mathbf{V}^{(t)} +\rho_t \mathbf{D}^{(t)})\,\, &\mathbf{L}^{\star} &\!=\! \mathbf{H}^{-1/2} P_r(\mathbf{H}^{1/2} (\widehat{\mathbf{W}}- \mathbf{S}^{(t+1)})) \notag \\
&\mathbf{D}^{(t+1)} \!=\! P_\mathcal{S}(\mathbf{S}^{(t+1)} + \mathbf{V}^{(t)}/\rho_t)\,\,& \mathbf{V}^{(t+1)}& \!=\! \mathbf{V}^{(t)} + \rho_t(\mathbf{S}^{(t+1)}-\mathbf{D}^{(t+1)}).
\end{alignat}

The initial $\rho_0 = 0.1$. The $\rho$-update for ADMM depends on the support change similar to what was proposed by \cite{meng2024alps}. The \slr decomposition is more \say{sensitive} to increasing $\rho$ aggressively compared to pure pruning in the works of \cite{meng2024alps}. We use the following $\rho$ update rules.
We update $\rho$ every $10$ iteration based on a step function that depends on the current value of $\rho_t$ and $s_t := |\operatorname{Supp}\left(\mathbf{D}^{(t)}\right) \Delta \operatorname{Supp}\left(\mathbf{D}^{(t-10)}\right)|$, which represents the number of elements in the symmetric difference between $\operatorname{Supp}\left(\mathbf{D}^{(t)}\right)$ and $\operatorname{Supp}\left(\mathbf{D}^{(t-10)}\right)$. Specifically, we set
\begin{equation}
\rho_{t+1} = \left\{\begin{array}{ll}
1.1\rho_t & \text{ if } s_t \ge 0.1k, \\
1.05\rho_t & \text{ if } s_t \ge 0.005k, \\
1.02\rho_t & \text{ if } s_t \ge 0.5. \\
\end{array} \right.
\end{equation}
It is worth noting that the algorithm can converge significantly faster if we set these parameters to the ones proposed by \cite{meng2024alps} (ADMM for pruning) but the solution quality can be slightly compromised.

\paragraph{Implementation Details of transformer matching (\texttt{TM})} Given a transformer $T_i$ with input activations $\mathbf{X_i}$ (obtained from the outputs of the previously compressed transformer block $T_{i-1}$, we start by creating a copy of $T_i$, termed $T_i^{\textbf{ori}}$. We then compress $T_i$ layers using an \slr method. We now replace dense layers with LoRA layers that contain new linear sparse layers and low-rank components $\mathbf{A, B}$. We set all parameters in transformer block $T_i$ to be trainable and minimize using Adam the loss $\|T_i(\mathbf{X_i}) - T_i^{(ori)}(\mathbf{X_i})\|_F^2$. The input activations fed into subsequent transformer blocks are $T_i^{(TM)}(\mathbf{X_i})$, where $T_i^{(TM)}$ is the transformer block after \slr decomposition and \texttt{TM} refinement steps.

For \texttt{TM} step, we employ the Adam optimizer with PyTorch’s default hyperparameters. We use 20 epochs (on the $128$ calibration data points selected for compression). The batch size used is $8$. The learning rate is $2e^{-5}$ using a Cosine Annealing Scheduler with $\eta_\text{min}=4e^{-6}$.

\paragraph{Baseline Implementation Details}
Below are the implementation specifications for:
\begin{itemize}[noitemsep,topsep=0pt,parsep=0pt,partopsep=0pt, leftmargin=*]
\item \textbf{OATS:} We adopt the official implementation from \cite{zhang2024oats} (accessible via \href{https://github.com/stephenqz/OATS}{GitHub}) and apply the default hyperparameters and 80 alternating minimization steps.
\item \textbf{HASSLE-free-SparseGPT:} We adopt the official implementation from \cite{makni2025hasslefree} (accessible via \href{https://github.com/mazumder-lab/HASSLE-free/tree/main}{GitHub}) and provide an improved implementation that uses the closed-form solution \cref{eq:L-update} for the low-rank fitting step. We apply the default hyperparameters and 80 alternating minimization steps.
\item \textbf{HASSLE-free-ALPS:} We adopt the official implementation from \cite{makni2025hasslefree} (accessible via \href{https://github.com/mazumder-lab/HASSLE-free/tree/main}{GitHub}) and provide an improved implementation that uses the closed-form solution \cref{eq:L-update} for the low-rank fitting step. We apply the default hyperparameters and 80 alternating minimization steps.
\end{itemize}
In \cref{tab:compression-rho}, we use the values reported in \cite{makni2025hasslefree}. For all other reported values, instead of minimizing $\|\mathbf{X}(\mathbf{W}-\mathbf{M})\|_F \; \text{s.t.} \, \operatorname{rank}(\mathbf{M}) \leq r$ by reparameterizing $\mathbf{M}=\mathbf{UV}^\top$ and optimizing with gradient-descent on $\mathbf{U}$ and $\mathbf{V}$ as proposed by the authors, we use our improved implementation of HASSLE-free with closed-form solution \cref{eq:L-update}. This results in significant speedup improvements. A slight improvement in LLM evaluation benchmarks is also sometimes observed using the improved implementation. This is expected because gradient-descent on $\mathbf{U}$ and $\mathbf{V}$ \textbf{approximately} solves the reduced-rank regression problem, whereas the closed-form solution is an optimal solution.

\cref{tab:our-hassle-free} shows an extract of the differences between the implementation of HASSLE-free proposed in \cite{makni2025hasslefree} and ours (using closed-form solution for low-rank update). Moreover, the original paper reports a compression runtime (of a Llama3-8B under a 2:4+64LR configuration) of 20.13 hours using a single A100 80GB GPU, whereas we report a compression runtime (for the same setup) of 15.71 hours in \cref{fig:transformer-matching-universality} (using a single A100 80GB GPU) thanks to the efficiency of the closed-form solution. It is worth noting that \ourmethod and \texttt{3BASiL-TM} are still over 7 times and 3 times, respectively, faster than HASSLE-free-ALPS, even when using the improved implementation for HASSLE-free.

\begin{table}
\centering
\resizebox{\textwidth}{!}{%
\centering
\begin{tabular}{llcccccccccccccc}
\toprule
\multirow{2.25}{*}{\textbf{Model}} & \multirow{2.25}{*}{\textbf{Algorithm}} && \multicolumn{3}{c}{\textbf{Perplexity ($\downarrow$)}} && \multicolumn{9}{c}{\textbf{Zero-shot ($\uparrow$)}} \\
\cmidrule(lr){4-6} \cmidrule(lr){8-16}
&&& \textbf{C4} & \textbf{WT2} & \textbf{PTB}
&& \textbf{PIQA} & \textbf{HS} & \textbf{ARC-E} & \textbf{ARC-C} & \textbf{WG} & \textbf{RTE} & \textbf{OQA} & \textbf{BoolQ} & \textbf{Avg} \\
\midrule
\multirow{4}{*}{Llama3-8B}
& Hf-SparseGPT-original && 18.06 & 12.66 & 18.66 && 74.86 & 64.77 & 63.85 & 37.37 & 69.22 & 56.68 & 36.40 & 76.12 & 59.91 \\
& Hf-SparseGPT-ours && 17.77 & 12.38 & 18.71 && 74.81 & 65.04 & 66.16 & 38.57 & 70.09 & 54.87 & 38.40 & 77.71 & 60.71 \\
& EoRA-SparseGPT && 21.89 & 15.69 & 23.91 && 72.25 & 58.87 & 57.70 & 34.56 & 66.30 & 54.87 & 33.80 & 73.58 & 56.49 \\

& Hf-ALPS-original && 16.76 & 11.83 & 17.76 && 75.08 & 66.37 & 63.64 & 37.54 & 69.69 & 64.62 & 37.20 & 77.89 & 61.50 \\
& Hf-ALPS-ours && 16.15 & 11.38 & 16.71 && 75.19 & 67.10 & 64.44 & 38.91 & 69.53 & 59.93 & 39.40 & 78.38 & 61.61 \\
& EoRA-ALPS && {18.69} & {13.61} & {20.55} && {73.99} & {62.21} & {61.07} & {37.20} & {68.59} & {57.40} & {36.00} & {74.16} & {58.83} \\
\bottomrule
\end{tabular}
}
\vspace{0.5em}
\caption{Comparison of (original paper), our reproduced results (with improved implementation) and EoRA (reduces to HASSLE-free with alternating minimization steps set to 1) for Llama3-8B under the configuration (2:4 + 64). Perplexity (lower is better), Zero-shot accuracy (higher is better).}\label{tab:our-hassle-free}
\end{table}

%TODO add eora

\paragraph{LoRA Finetuning Details}
We follow a similar LoRA fine-tuning pipeline to the one introduced in \cite{guo2023lq}.
For LoRA fine-tuning use a learning rate of 2e-5 and a batch size of size 64 per step. The block size used is $1024$ tokens per batch. The effective batch size is obtained by using a physical batch size of 2 on GPU with 32 gradient accumulation steps before each weight update. Training is conducted on 10\% of the first shard of the C4 training dataset, which contains over 15 million tokens. We employ the Adam optimizer with PyTorch's default hyperparameters. A cosine learning rate scheduler is used, with a warm-up ratio of 0.03 and no weight decay applied.

\paragraph{Layer-wise reconstruction error of \ourmethod} \cref{fig:attn-layers-loss} and \cref{fig:mlp-layers-loss} show the objective of \cref{eq:original} attained by \ourmethod and other \slr methods for the first transformer block of a 2:4+64LR decomposition of a Llama3-8B model for Attention and MLP layers, respectively.

\begin{figure}[h]
\centering
\begin{subfigure}[b]{0.48\textwidth}
\centering
\includegraphics[width=\textwidth]{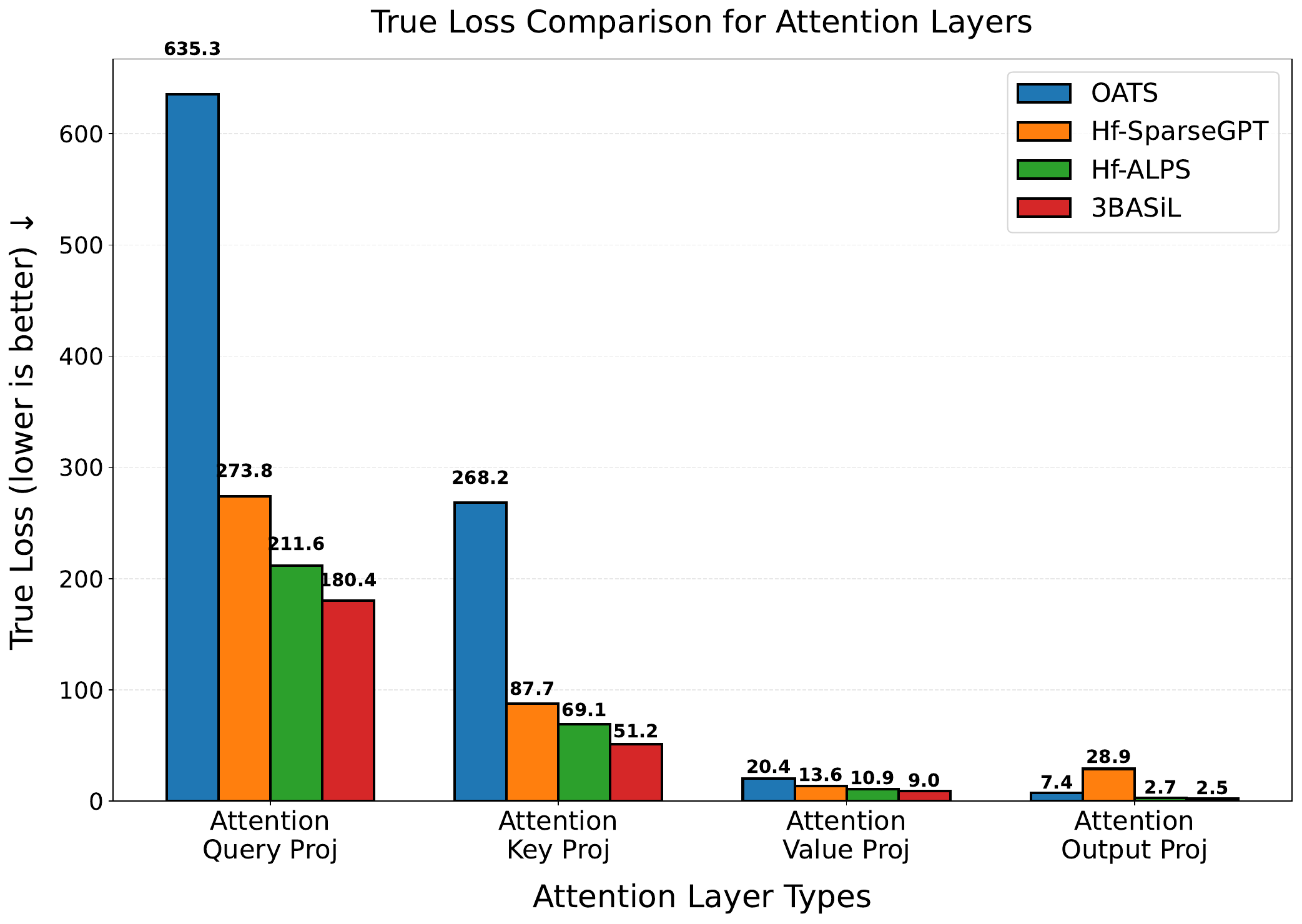}
\caption{True loss for attention layers (linear scale).}
\label{fig:attn-layers-loss}
\end{subfigure}
\hfill
\begin{subfigure}[b]{0.48\textwidth}
\centering
\includegraphics[width=\textwidth]{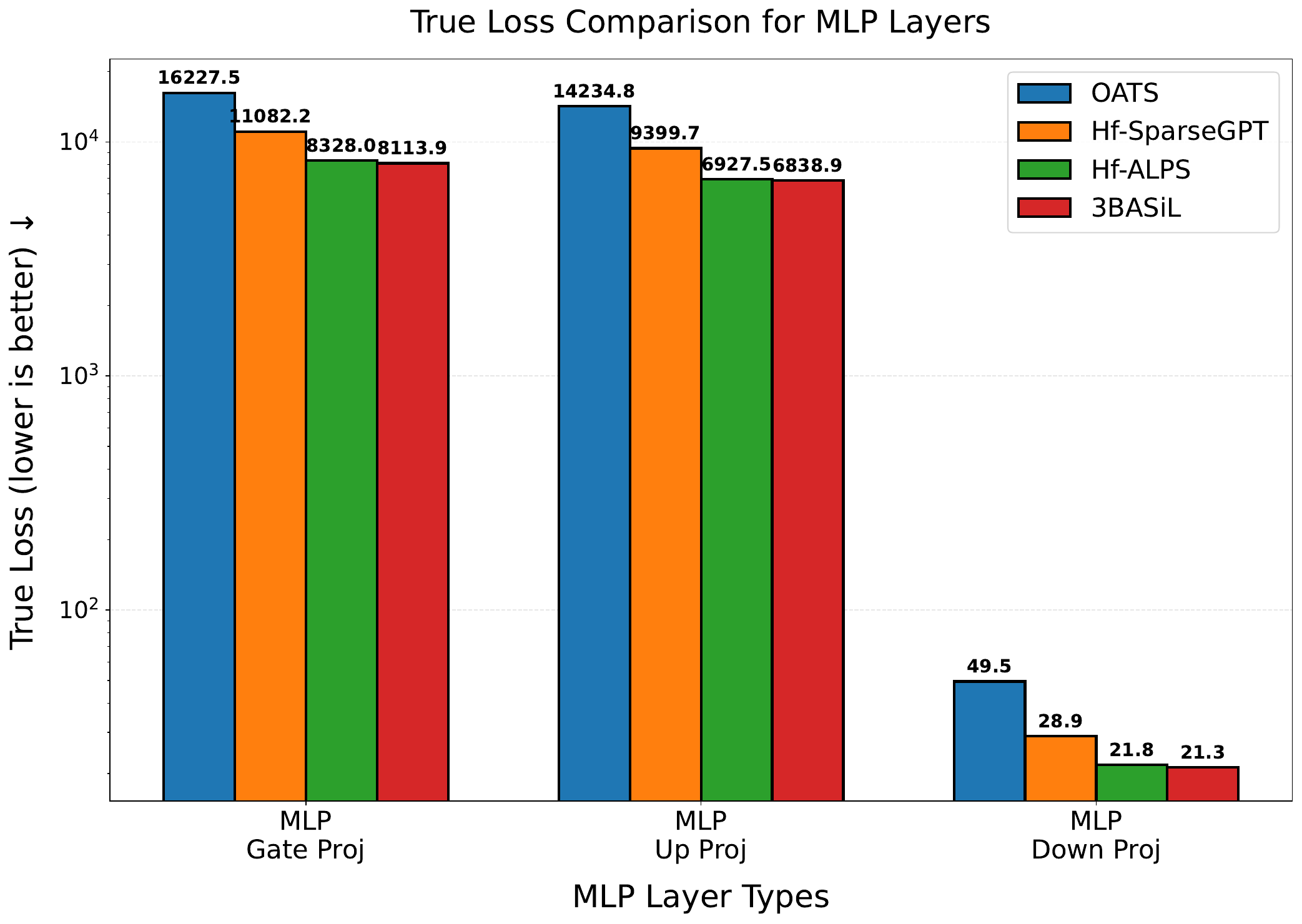}
\caption{True loss for MLP layers (log scale).}
\label{fig:mlp-layers-loss}
\end{subfigure}
\caption{Comparison of true loss values introduced in \cref{eq:original} across different \slr methods. Lower values indicate better optimization quality. \ourmethod consistently outperforms other methods, particularly for attention layers.}
\label{fig:loss-comparison}
\end{figure}
\newpage
\section{Additional Experimental Results}\label{sect:appendix-more-experiments}
We provide additional performance results considered in \cref{sec:experimental-results}. We compare different \slr algorithms and their \texttt{TM}-enhanced versions (apply \texttt{TM} as an add-on to the decomposition algorithm). In that case, we add the suffix \texttt{-TM} to the algorithm. We mark algorithms with \texttt{TM} in gray. We also study the results of the \slr decomposition after LoRA fine-tuning as described in \cref{sect:appendix-exp-details}. In that case, we add the prefix \texttt{LFT-} to the algorithm.

Example: \texttt{LFT-OATS-TM} denotes the results of \slr decomposition after (i) using OATS to obtain sparse and low-rank components, (ii) refine these decomposed components with \texttt{TM} and (iii) LoRA fine-tunes the model by using the low-rank components from the \slr decomposition as a smart initialization.
% !TEX root = ../main.tex

\begin{table}[H]
\centering
\resizebox{0.90\textwidth}{!}{%
\begin{tabular}{l|c|ccc|ccccccccc}
\toprule
\textbf{Method} & \textbf{Config} & \textbf{C4 $\downarrow$} & \textbf{WT2 $\downarrow$} & \textbf{PTB $\downarrow$} & \textbf{PIQA $\uparrow$} & \textbf{HS $\uparrow$} & \textbf{ARC-E $\uparrow$} & \textbf{ARC-C $\uparrow$} & \textbf{WG $\uparrow$} & \textbf{RTE $\uparrow$} & \textbf{OQA $\uparrow$} & \textbf{BoolQ $\uparrow$} & \textbf{Avg} \\
\midrule
\texttt{OATS} & \multirow{8}{*}{2:8+64LR} & 640.86 & 605.20 & 779.86 & 52.01 & 27.67 & 28.66 & \textbf{23.12} & 49.96 & 52.71 & 25.00 & 37.83 & 37.12 \\
\cellcolor{gray!15}\texttt{OATS-TM} &  & \cellcolor{gray!15}116.29 & \cellcolor{gray!15}99.92 & \cellcolor{gray!15}126.06 & \cellcolor{gray!15}55.55 & \cellcolor{gray!15}28.78 & \cellcolor{gray!15}31.44 & \cellcolor{gray!15}20.56 & \cellcolor{gray!15}51.38 & \cellcolor{gray!15}52.71 & \cellcolor{gray!15}26.80 & \cellcolor{gray!15}46.79 & \cellcolor{gray!15}39.25 \\
\texttt{Hassle-free-SparseGPT} &  & 162.45 & 134.21 & 170.12 & 54.19 & 28.28 & 31.40 & 22.01 & 49.17 & 52.71 & 26.80 & 57.25 & 40.23 \\
\cellcolor{gray!15}\texttt{Hassle-free-SparseGPT-TM} &  & \cellcolor{gray!15}74.50 & \cellcolor{gray!15}67.59 & \cellcolor{gray!15}88.05 & \cellcolor{gray!15}57.24 & \cellcolor{gray!15}30.03 & \cellcolor{gray!15}33.12 & \cellcolor{gray!15}21.33 & \cellcolor{gray!15}52.33 & \cellcolor{gray!15}52.71 & \cellcolor{gray!15}26.20 & \cellcolor{gray!15}60.34 & \cellcolor{gray!15}41.66 \\
\texttt{Hassle-free-ALPS} &  & 107.14 & 94.71 & 124.17 & 55.39 & 29.98 & 32.07 & 20.82 & 52.72 & \textbf{53.07} & \textbf{27.20} & 49.27 & 40.06 \\
\cellcolor{gray!15}\texttt{Hassle-free-ALPS-TM} &  & \cellcolor{gray!15}58.30 & \cellcolor{gray!15}52.30 & \cellcolor{gray!15}73.97 & \cellcolor{gray!15}\textbf{58.81} & \cellcolor{gray!15}32.21 & \cellcolor{gray!15}34.93 & \cellcolor{gray!15}21.59 & \cellcolor{gray!15}52.80 & \cellcolor{gray!15}52.71 & \cellcolor{gray!15}26.20 & \cellcolor{gray!15}\textbf{62.20} & \cellcolor{gray!15}42.68 \\
\bottomrule
\end{tabular}%
}
\vspace{3pt}
\caption{One-shot (N:M Sparse + LR) decomposition performance for Llama-3.2-1B. For Perplexity, $(\downarrow)$ lower values are better. For zero-shot tasks, $(\uparrow)$ higher values are better.}
\end{table}
\begin{table}[H]
\centering
\resizebox{0.90\textwidth}{!}{%
\begin{tabular}{l|c|ccc|ccccccccc}
\toprule
\textbf{Method} & \textbf{Config} & \textbf{C4 $\downarrow$} & \textbf{WT2 $\downarrow$} & \textbf{PTB $\downarrow$} & \textbf{PIQA $\uparrow$} & \textbf{HS $\uparrow$} & \textbf{ARC-E $\uparrow$} & \textbf{ARC-C $\uparrow$} & \textbf{WG $\uparrow$} & \textbf{RTE $\uparrow$} & \textbf{OQA $\uparrow$} & \textbf{BoolQ $\uparrow$} & \textbf{Avg} \\
\midrule
\texttt{OATS} & \multirow{8}{*}{2:8+64LR} & 531.47 & 494.31 & 674.71 & 52.50 & 27.33 & 28.16 & \textbf{23.29} & 49.57 & 52.71 & 26.60 & 39.60 & 37.47 \\
\cellcolor{gray!15}\texttt{OATS-TM} &  & \cellcolor{gray!15}100.87 & \cellcolor{gray!15}87.20 & \cellcolor{gray!15}120.98 & \cellcolor{gray!15}56.64 & \cellcolor{gray!15}29.01 & \cellcolor{gray!15}30.13 & \cellcolor{gray!15}20.65 & \cellcolor{gray!15}50.99 & \cellcolor{gray!15}52.71 & \cellcolor{gray!15}26.00 & \cellcolor{gray!15}62.11 & \cellcolor{gray!15}41.03 \\
\texttt{Hassle-free-SparseGPT} &  & 106.07 & 106.17 & 151.92 & 54.62 & 29.67 & 29.92 & 21.67 & 50.28 & 52.71 & 26.60 & 61.93 & 40.93 \\
\cellcolor{gray!15}\texttt{Hassle-free-SparseGPT-TM} &  & \cellcolor{gray!15}61.50 & \cellcolor{gray!15}56.02 & \cellcolor{gray!15}90.37 & \cellcolor{gray!15}58.98 & \cellcolor{gray!15}32.53 & \cellcolor{gray!15}33.75 & \cellcolor{gray!15}22.10 & \cellcolor{gray!15}51.78 & \cellcolor{gray!15}52.71 & \cellcolor{gray!15}26.40 & \cellcolor{gray!15}62.11 & \cellcolor{gray!15}42.55 \\
\texttt{Hassle-free-ALPS} &  & 69.96 & 65.34 & 108.68 & 57.34 & 32.59 & 33.59 & 20.82 & 50.67 & 52.71 & \textbf{27.00} & 62.26 & 42.12 \\
\cellcolor{gray!15}\texttt{Hassle-free-ALPS-TM} &  & \cellcolor{gray!15}46.12 & \cellcolor{gray!15}44.03 & \cellcolor{gray!15}\textbf{61.25} & \cellcolor{gray!15}\textbf{61.26} & \cellcolor{gray!15}36.36 & \cellcolor{gray!15}36.83 & \cellcolor{gray!15}23.12 & \cellcolor{gray!15}52.80 & \cellcolor{gray!15}52.71 & \cellcolor{gray!15}25.00 & \cellcolor{gray!15}62.51 & \cellcolor{gray!15}43.82 \\
\bottomrule
\end{tabular}%
}
\vspace{3pt}
\caption{One-shot (N:M Sparse + LR) decomposition performance for Llama-3.2-3B. For Perplexity, $(\downarrow)$ lower values are better. For zero-shot tasks, $(\uparrow)$ higher values are better.}
\end{table}
\begin{table}[H]
\centering
\resizebox{0.98\textwidth}{!}{%
\begin{tabular}{l|c|ccc|ccccccccc}
\toprule
\textbf{Method} & \textbf{Config} & \textbf{C4 $\downarrow$} & \textbf{WT2 $\downarrow$} & \textbf{PTB $\downarrow$} & \textbf{PIQA $\uparrow$} & \textbf{HS $\uparrow$} & \textbf{ARC-E $\uparrow$} & \textbf{ARC-C $\uparrow$} & \textbf{WG $\uparrow$} & \textbf{RTE $\uparrow$} & \textbf{OQA $\uparrow$} & \textbf{BoolQ $\uparrow$} & \textbf{Avg} \\
\midrule
\texttt{OATS} & \multirow{8}{*}{2:8+64LR} & 424.55 & 431.81 & 590.88 & 51.63 & 28.11 & 27.61 & 23.72 & 49.72 & \textbf{52.71} & 27.00 & 38.26 & 37.34 \\
\cellcolor{gray!15}\texttt{OATS-TM} &  & \cellcolor{gray!15}73.42 & \cellcolor{gray!15}64.21 & \cellcolor{gray!15}100.59 & \cellcolor{gray!15}58.43 & \cellcolor{gray!15}31.23 & \cellcolor{gray!15}32.20 & \cellcolor{gray!15}20.22 & \cellcolor{gray!15}51.38 & \cellcolor{gray!15}\textbf{52.71} & \cellcolor{gray!15}26.00 & \cellcolor{gray!15}62.29 & \cellcolor{gray!15}41.81 \\
\texttt{Hassle-free-SparseGPT} &  & 88.39 & 96.61 & 109.71 & 54.95 & 30.77 & 31.48 & 20.65 & 50.75 & \textbf{52.71} & 26.20 & 61.59 & 41.14 \\
\cellcolor{gray!15}\texttt{Hassle-free-SparseGPT-TM} &  & \cellcolor{gray!15}46.24 & \cellcolor{gray!15}42.75 & \cellcolor{gray!15}71.66 & \cellcolor{gray!15}61.48 & \cellcolor{gray!15}36.76 & \cellcolor{gray!15}36.74 & \cellcolor{gray!15}23.12 & \cellcolor{gray!15}53.51 & \cellcolor{gray!15}\textbf{52.71} & \cellcolor{gray!15}27.60 & \cellcolor{gray!15}63.15 & \cellcolor{gray!15}44.38 \\
\texttt{Hassle-free-ALPS} &  & 60.16 & 56.03 & 77.11 & 57.94 & 34.92 & 34.64 & 21.67 & 54.62 & \textbf{52.71} & 27.80 & 56.12 & 42.55 \\
\bottomrule
\end{tabular}%
}
\vspace{3pt}
\caption{One-shot (N:M Sparse + LR) decomposition performance for Meta-Llama-3-8B. For Perplexity, $(\downarrow)$ lower values are better. For zero-shot tasks, $(\uparrow)$ higher values are better.}
\end{table}

\begin{table}[H]
\centering
\resizebox{0.98\textwidth}{!}{%
\begin{tabular}{l|c|ccc|ccccccccc}
\toprule
\textbf{Method} & \textbf{Config} & \textbf{C4 $\downarrow$} & \textbf{WT2 $\downarrow$} & \textbf{PTB $\downarrow$} & \textbf{PIQA $\uparrow$} & \textbf{HS $\uparrow$} & \textbf{ARC-E $\uparrow$} & \textbf{ARC-C $\uparrow$} & \textbf{WG $\uparrow$} & \textbf{RTE $\uparrow$} & \textbf{OQA $\uparrow$} & \textbf{BoolQ $\uparrow$} & \textbf{Avg} \\
\midrule
\texttt{LFT-OATS} & \multirow{8}{*}{2:8+64LR} & 71.41 & 63.58 & 95.05 & 56.58 & 29.22 & 30.72 & 22.10 & 53.83 & 52.71 & 23.00 & 61.10 & 41.16 \\
\cellcolor{gray!15}\texttt{LFT-OATS-TM} &  & \cellcolor{gray!15}53.01 & \cellcolor{gray!15}47.68 & \cellcolor{gray!15}68.98 & \cellcolor{gray!15}59.09 & \cellcolor{gray!15}31.07 & \cellcolor{gray!15}36.15 & \cellcolor{gray!15}21.84 & \cellcolor{gray!15}52.49 & \cellcolor{gray!15}52.71 & \cellcolor{gray!15}24.60 & \cellcolor{gray!15}62.14 & \cellcolor{gray!15}42.51 \\
\texttt{LFT-Hassle-free-SparseGPT} &  & 51.19 & 47.09 & 63.87 & 59.03 & 31.16 & 35.14 & 22.18 & 50.83 & 52.71 & 26.60 & 60.70 & 42.29 \\
\cellcolor{gray!15}\texttt{LFT-Hassle-free-SparseGPT-TM} &  & \cellcolor{gray!15}42.85 & \cellcolor{gray!15}39.46 & \cellcolor{gray!15}55.53 & \cellcolor{gray!15}\textbf{60.55} & \cellcolor{gray!15}33.37 & \cellcolor{gray!15}35.27 & \cellcolor{gray!15}22.18 & \cellcolor{gray!15}50.36 & \cellcolor{gray!15}52.71 & \cellcolor{gray!15}28.00 & \cellcolor{gray!15}62.14 & \cellcolor{gray!15}43.07 \\
\texttt{LFT-Hassle-free-ALPS} &  & 44.15 & 41.02 & 57.45 & 59.36 & 33.22 & 34.68 & 22.18 & 51.93 & 52.71 & 26.60 & 60.95 & 42.70 \\
\cellcolor{gray!15}\texttt{LFT-Hassle-free-ALPS-TM} &  & \cellcolor{gray!15}37.54 & \cellcolor{gray!15}35.50 & \cellcolor{gray!15}50.15 & \cellcolor{gray!15}60.28 & \cellcolor{gray!15}35.80 & \cellcolor{gray!15}37.50 & \cellcolor{gray!15}23.63 & \cellcolor{gray!15}55.09 & \cellcolor{gray!15}52.71 & \cellcolor{gray!15}27.20 & \cellcolor{gray!15}62.14 & \cellcolor{gray!15}44.29 \\
\bottomrule
\end{tabular}%
}
\vspace{3pt}
\caption{(N:M Sparse + LR) decomposition performance for Llama-3.2-1B after LoRa Fine-Tuning (LFT). For Perplexity, $(\downarrow)$ lower values are better. For zero-shot tasks, $(\uparrow)$ higher values are better.}
\end{table}
\begin{table}[H]
\centering
\resizebox{0.98\textwidth}{!}{%
\begin{tabular}{l|c|ccc|ccccccccc}
\toprule
\textbf{Method} & \textbf{Config} & \textbf{C4 $\downarrow$} & \textbf{WT2 $\downarrow$} & \textbf{PTB $\downarrow$} & \textbf{PIQA $\uparrow$} & \textbf{HS $\uparrow$} & \textbf{ARC-E $\uparrow$} & \textbf{ARC-C $\uparrow$} & \textbf{WG $\uparrow$} & \textbf{RTE $\uparrow$} & \textbf{OQA $\uparrow$} & \textbf{BoolQ $\uparrow$} & \textbf{Avg} \\
\midrule
\texttt{LFT-OATS} & \multirow{8}{*}{2:8+64LR} & 48.53 & 44.51 & 62.21 & 58.27 & 32.84 & 36.95 & 22.78 & 52.96 & 52.71 & 25.60 & 61.44 & 42.94 \\
\cellcolor{gray!15}\texttt{LFT-OATS-TM} &  & \cellcolor{gray!15}100.87 & \cellcolor{gray!15}87.20 & \cellcolor{gray!15}120.98 & \cellcolor{gray!15}56.64 & \cellcolor{gray!15}29.01 & \cellcolor{gray!15}30.13 & \cellcolor{gray!15}20.65 & \cellcolor{gray!15}50.99 & \cellcolor{gray!15}52.71 & \cellcolor{gray!15}26.00 & \cellcolor{gray!15}62.11 & \cellcolor{gray!15}41.03 \\
\texttt{LFT-Hassle-free-SparseGPT} &  & 35.91 & 33.44 & 46.64 & 61.48 & 37.33 & 36.70 & 23.46 & 51.54 & 52.35 & 26.20 & 62.17 & 43.90 \\
\cellcolor{gray!15}\texttt{LFT-Hassle-free-SparseGPT-TM} &  & \cellcolor{gray!15}31.46 & \cellcolor{gray!15}30.77 & \cellcolor{gray!15}42.12 & \cellcolor{gray!15}62.79 & \cellcolor{gray!15}40.19 & \cellcolor{gray!15}38.76 & \cellcolor{gray!15}23.63 & \cellcolor{gray!15}54.78 & \cellcolor{gray!15}53.07 & \cellcolor{gray!15}29.00 & \cellcolor{gray!15}58.32 & \cellcolor{gray!15}45.07 \\
\texttt{LFT-Hassle-free-ALPS} &  & 31.43 & 29.50 & 44.39 & 63.38 & 40.65 & 39.73 & 23.46 & 55.72 & 52.71 & 28.40 & 46.06 & 43.76 \\
\cellcolor{gray!15}\texttt{LFT-Hassle-free-ALPS-TM} &  & \cellcolor{gray!15}28.22 & \cellcolor{gray!15}26.81 & \cellcolor{gray!15}\textbf{36.23} & \cellcolor{gray!15}65.02 & \cellcolor{gray!15}43.69 & \cellcolor{gray!15}40.78 & \cellcolor{gray!15}23.46 & \cellcolor{gray!15}55.72 & \cellcolor{gray!15}52.71 & \cellcolor{gray!15}28.40 & \cellcolor{gray!15}60.46 & \cellcolor{gray!15}46.28 \\
\bottomrule
\end{tabular}%
}
\vspace{3pt}
\caption{(N:M Sparse + LR) decomposition performance for Llama-3.2-3B after LoRa Fine-Tuning (LFT). For Perplexity, $(\downarrow)$ lower values are better. For zero-shot tasks, $(\uparrow)$ higher values are better.}
\end{table}
\begin{table}[H]
\centering
\resizebox{0.98\textwidth}{!}{%
\begin{tabular}{l|c|ccc|ccccccccc}
\toprule
\textbf{Method} & \textbf{Config} & \textbf{C4 $\downarrow$} & \textbf{WT2 $\downarrow$} & \textbf{PTB $\downarrow$} & \textbf{PIQA $\uparrow$} & \textbf{HS $\uparrow$} & \textbf{ARC-E $\uparrow$} & \textbf{ARC-C $\uparrow$} & \textbf{WG $\uparrow$} & \textbf{RTE $\uparrow$} & \textbf{OQA $\uparrow$} & \textbf{BoolQ $\uparrow$} & \textbf{Avg} \\
\midrule
\texttt{LFT-OATS} & \multirow{8}{*}{2:8+64LR} & 37.46 & 31.49 & 49.71 & 62.24 & 36.88 & 38.72 & 24.23 & 51.93 & 52.71 & 27.20 & 62.02 & 44.49 \\
\cellcolor{gray!15}\texttt{LFT-OATS-TM} &  & \cellcolor{gray!15}28.23 & \cellcolor{gray!15}23.86 & \cellcolor{gray!15}36.25 & \cellcolor{gray!15}65.61 & \cellcolor{gray!15}43.81 & \cellcolor{gray!15}42.09 & \cellcolor{gray!15}25.26 & \cellcolor{gray!15}52.64 & \cellcolor{gray!15}52.71 & \cellcolor{gray!15}30.60 & \cellcolor{gray!15}64.13 & \cellcolor{gray!15}47.11 \\
\texttt{LFT-Hassle-free-SparseGPT} &  & 28.80 & 24.47 & 33.94 & 62.35 & 43.30 & 40.28 & 24.91 & 54.85 & 53.07 & 29.40 & 64.46 & 46.58 \\
\cellcolor{gray!15}\texttt{LFT-Hassle-free-SparseGPT-TM} &  & \cellcolor{gray!15}24.89 & \cellcolor{gray!15}22.14 & \cellcolor{gray!15}31.65 & \cellcolor{gray!15}66.00 & \cellcolor{gray!15}48.49 & \cellcolor{gray!15}43.77 & \cellcolor{gray!15}26.37 & \cellcolor{gray!15}58.25 & \cellcolor{gray!15}52.71 & \cellcolor{gray!15}29.80 & \cellcolor{gray!15}66.27 & \cellcolor{gray!15}48.96 \\
\texttt{LFT-Hassle-free-ALPS} &  & 25.31 & 21.97 & 31.99 & 66.38 & 48.42 & 43.73 & 26.62 & 59.43 & 53.43 & 29.60 & 47.16 & 46.85 \\
\cellcolor{gray!15}\texttt{LFT-Hassle-free-ALPS-TM} &  & \cellcolor{gray!15}22.85 & \cellcolor{gray!15}20.23 & \cellcolor{gray!15}\textbf{28.00} & \cellcolor{gray!15}67.95 & \cellcolor{gray!15}52.18 & \cellcolor{gray!15}\textbf{46.84} & \cellcolor{gray!15}\textbf{27.99} & \cellcolor{gray!15}60.54 & \cellcolor{gray!15}53.43 & \cellcolor{gray!15}31.20 & \cellcolor{gray!15}61.44 & \cellcolor{gray!15}50.20 \\
\bottomrule
\end{tabular}%
}
\vspace{3pt}
\caption{(N:M Sparse + LR) decomposition performance for Meta-Llama-3-8B after LoRa Fine-Tuning (LFT). For Perplexity, $(\downarrow)$ lower values are better. For zero-shot tasks, $(\uparrow)$ higher values are better.}
\end{table}

\begin{table}[htbp]
\centering
\resizebox{\textwidth}{!}{%
\begin{tabular}{l|c|ccc|ccccccccc}
\toprule
\textbf{Method} & \textbf{Config} & \textbf{C4 $\downarrow$} & \textbf{WT2 $\downarrow$} & \textbf{PTB $\downarrow$} & \textbf{PIQA $\uparrow$} & \textbf{HS $\uparrow$} & \textbf{ARC-E $\uparrow$} & \textbf{ARC-C $\uparrow$} & \textbf{WG $\uparrow$} & \textbf{RTE $\uparrow$} & \textbf{OQA $\uparrow$} & \textbf{BoolQ $\uparrow$} & \textbf{Avg} \\
\midrule
\cellcolor{darkgreen!10}3BASiL & & \cellcolor{darkgreen!10}25.31 & \cellcolor{darkgreen!10}20.59 & \cellcolor{darkgreen!10}28.33 & \cellcolor{darkgreen!10}71.27 & \cellcolor{darkgreen!10}55.56 & \cellcolor{darkgreen!10}54.25 & \cellcolor{darkgreen!10}30.80 & \cellcolor{darkgreen!10}64.56 & \cellcolor{darkgreen!10}53.43 & \cellcolor{darkgreen!10}33.20 & \cellcolor{darkgreen!10}72.91 & \cellcolor{darkgreen!10}54.50 \\
\cellcolor{darkgreen!10}3BASiL+OWL & & \cellcolor{darkgreen!10}\textbf{23.21} & \cellcolor{darkgreen!10}\textbf{19.54} & \cellcolor{darkgreen!10}\textbf{27.09} & \cellcolor{darkgreen!10}\textbf{71.71} & \cellcolor{darkgreen!10}\textbf{57.90} & \cellcolor{darkgreen!10}\textbf{56.69} & \cellcolor{darkgreen!10}\textbf{33.53} & \cellcolor{darkgreen!10}\textbf{67.56} & \cellcolor{darkgreen!10}\textbf{54.15} & \cellcolor{darkgreen!10}\textbf{33.40} & \cellcolor{darkgreen!10}\textbf{77.86} & \cellcolor{darkgreen!10}\textbf{56.60} \\
\cmidrule{1-1}\cmidrule{3-14}
\cellcolor{darkgreen!10}3BASiL-TM & & \cellcolor{darkgreen!10}19.55 & \cellcolor{darkgreen!10}16.44 & \cellcolor{darkgreen!10}21.64 & \cellcolor{darkgreen!10}\textbf{74.16} & \cellcolor{darkgreen!10}59.76 & \cellcolor{darkgreen!10}58.04 & \cellcolor{darkgreen!10}\textbf{33.28} & \cellcolor{darkgreen!10}\textbf{66.38} & \cellcolor{darkgreen!10}\textbf{58.12} & \cellcolor{darkgreen!10}35.20 & \cellcolor{darkgreen!10}\textbf{73.15} & \cellcolor{darkgreen!10}\textbf{57.26} \\
\bottomrule
\end{tabular}%
}
\vspace{3pt}
\caption{Impact of OWL on \ourmethod for (Unstructured + 64) decompositions of Meta-Llama-3-8B.}
\label{tab:owl_comparison}
\end{table}

% %%%%%%%%%%%%%%%%%%%%%%%%%%%%%%%%%%%%%%%%%%%%%%%%%%%%%%%%%%%%
\clearpage
\section*{NeurIPS Paper Checklist}
\begin{enumerate}

\item {\bf Claims}
    \item[] Question: Do the main claims made in the abstract and introduction accurately reflect the paper's contributions and scope?
    \item[] Answer: \answerYes{} % Replace by \answerYes{}, \answerNo{}, or \answerNA{}.
    \item[] Justification: We conclude the introduction with a paragraph that explicitly outlines the paper’s contributions and scope.

\item {\bf Limitations}
    \item[] Question: Does the paper discuss the limitations of the work performed by the authors?
    \item[] Answer: \answerYes{} % Replace by \answerYes{}, \answerNo{}, or \answerNA{}.
    \item[] Justification: We discuss the limitation of our work at the end of Section \ref{sec:conclusion}.

\item {\bf Theory assumptions and proofs}
    \item[] Question: For each theoretical result, does the paper provide the full set of assumptions and a complete (and correct) proof?
    \item[] Answer: \answerYes{} % Replace by \answerYes{}, \answerNo{}, or \answerNA{}.
    \item[] Justification: We clearly state all assumptions in Theorem \ref{thm:admm} and provide a rigorous proof in Appendix \ref{sect:proofadmm}.

    \item {\bf Experimental result reproducibility}
    \item[] Question: Does the paper fully disclose all the information needed to reproduce the main experimental results of the paper to the extent that it affects the main claims and/or conclusions of the paper (regardless of whether the code and data are provided or not)?
    \item[] Answer: \answerYes{} % Replace by \answerYes{}, \answerNo{}, or \answerNA{}.
    \item[] Justification: We present a detailed description of the proposed 3-Block ADMM algorithm, including update rules and computational procedures, in Section~\ref{sec:3blockadmm}, and describe the Transformer matching procedure in Section~\ref{sec:transmatch}. Additional implementation details necessary for reproducing our results are provided in Appendix~\ref{sect:appendix-exp-details}.

\item {\bf Open access to data and code}
    \item[] Question: Does the paper provide open access to the data and code, with sufficient instructions to faithfully reproduce the main experimental results, as described in supplemental material?
    \item[] Answer: \answerYes{} % Replace by \answerYes{}, \answerNo{}, or \answerNA{}.
    \item[] Justification:  We will release the codes if the paper is accepted.

\item {\bf Experimental setting/details}
    \item[] Question: Does the paper specify all the training and test details (e.g., data splits, hyperparameters, how they were chosen, type of optimizer, etc.) necessary to understand the results?
    \item[] Answer: \answerYes{} % Replace by \answerYes{}, \answerNo{}, or \answerNA{}.
    \item[] Justification: We provide detailed training and evaluation settings for both our proposed pipeline and the baseline methods in Appendix~\ref{sect:appendix-exp-details}.

\item {\bf Experiment statistical significance}
    \item[] Question: Does the paper report error bars suitably and correctly defined or other appropriate information about the statistical significance of the experiments?
    \item[] Answer: \answerNo{}
    \item[] Justification: While we aimed to provide rigorous evaluation, we were constrained by computational resources and thus could not include statistical significance measures. We have, however, ensured consistent settings and fair comparisons across all baselines.

\item {\bf Experiments compute resources}
    \item[] Question: For each experiment, does the paper provide sufficient information on the computer resources (type of compute workers, memory, time of execution) needed to reproduce the experiments?
    \item[] Answer: \answerYes{} % Replace by \answerYes{}, \answerNo{}, or \answerNA{}.
    \item[] Justification: We provide details of the computational resources used for our experiments in Appendix \ref{sect:appendix-exp-details}.
    
\item {\bf Code of ethics}
    \item[] Question: Does the research conducted in the paper conform, in every respect, with the NeurIPS Code of Ethics \url{https://neurips.cc/public/EthicsGuidelines}?
    \item[] Answer: \answerYes{} % Replace by \answerYes{}, \answerNo{}, or \answerNA{}.
    \item[] Justification: We have carefully reviewed the NeurIPS Code of Ethics and confirm that all research presented in this paper adheres to its principles.

\item {\bf Broader impacts}
    \item[] Question: Does the paper discuss both potential positive societal impacts and negative societal impacts of the work performed?
    \item[] Answer: \answerNA{} % Replace by \answerYes{}, \answerNo{}, or \answerNA{}.
    \item[] Justification: To the best of our knowledge, our work has no societal impact.
    
\item {\bf Safeguards}
    \item[] Question: Does the paper describe safeguards that have been put in place for responsible release of data or models that have a high risk for misuse (e.g., pretrained language models, image generators, or scraped datasets)?
    \item[] Answer: \answerNA{}  % Replace by \answerYes{}, \answerNo{}, or \answerNA{}.
    \item[] Justification: To the best of our knowledge, our work poses no such risks.

\item {\bf Licenses for existing assets}
    \item[] Question: Are the creators or original owners of assets (e.g., code, data, models), used in the paper, properly credited and are the license and terms of use explicitly mentioned and properly respected?
    \item[] Answer: \answerYes{} % Replace by \answerYes{}, \answerNo{}, or \answerNA{}.
    \item[] Justification: At the start of Section~\ref{sec:experimental-results}, we reference all datasets and models involved in our experiments. The sources of the code used are listed in Appendix~\ref{sect:appendix-exp-details}.

\item {\bf New assets}
    \item[] Question: Are new assets introduced in the paper well documented and is the documentation provided alongside the assets?
    \item[] Answer: \answerNA{} % Replace by \answerYes{}, \answerNo{}, or \answerNA{}.
    \item[] Justification: Our paper does not release new assets.

\item {\bf Crowdsourcing and research with human subjects}
    \item[] Question: For crowdsourcing experiments and research with human subjects, does the paper include the full text of instructions given to participants and screenshots, if applicable, as well as details about compensation (if any)? 
    \item[] Answer:  \answerNA{}  % Replace by \answerYes{}, \answerNo{}, or \answerNA{}.
    \item[] Justification: Our paper does not involve crowdsourcing nor research with human subjects.

\item {\bf Institutional review board (IRB) approvals or equivalent for research with human subjects}
    \item[] Question: Does the paper describe potential risks incurred by study participants, whether such risks were disclosed to the subjects, and whether Institutional Review Board (IRB) approvals (or an equivalent approval/review based on the requirements of your country or institution) were obtained?
    \item[] Answer: \answerNA{} % Replace by \answerYes{}, \answerNo{}, or \answerNA{}.
    \item[] Justification:  Our paper does not involve crowdsourcing nor research with human subjects.

\item {\bf Declaration of LLM usage}
    \item[] Question: Does the paper describe the usage of LLMs if it is an important, original, or non-standard component of the core methods in this research? Note that if the LLM is used only for writing, editing, or formatting purposes and does not impact the core methodology, scientific rigorousness, or originality of the research, declaration is not required.
    %this research? 
    \item[] Answer: \answerNA{} % Replace by \answerYes{}, \answerNo{}, or \answerNA{}.   
    \item[] Justification: Although our work focuses on pruning LLMs, the core methods proposed do not involve LLMs as important, original, or non-standard components of the algorithm itself.

\end{enumerate}

\end{document}